\documentclass{article}

\usepackage{PRIMEarxiv}

\usepackage[utf8]{inputenc} 
\usepackage[T1]{fontenc}    
\usepackage{hyperref}       
\usepackage{url}            
\usepackage{booktabs}       
\usepackage{amsfonts}       
\usepackage{nicefrac}       
\usepackage{microtype}      
\usepackage{lipsum}
\usepackage{fancyhdr}       
\usepackage{graphicx}       
\graphicspath{{media/}}     

\usepackage{amsmath}
\usepackage{algorithm}
\usepackage{algpseudocode}
\usepackage{enumitem}
\usepackage{subcaption}
\usepackage{graphicx}
\usepackage{stfloats}       
\usepackage{hyperref}       
\usepackage{url}            
\usepackage{booktabs}       
\usepackage{amsfonts}       
\usepackage{nicefrac}       
\usepackage{microtype}      
\usepackage{xcolor}         

\title{Learning in Matching Games \\
with Bandit Feedback}

\author{
  Andreas Athanasopoulos \\
  University of Neuchatel\\
  Switzerland \\
  \texttt{andreas.athanasopoulos@unine.ch} \\
   \And
  Christos Dimitrakakis \\
  University of Neuchatel\\
  Switzerland \\
  \texttt{christos.dimitrakakis@gmail.com} \\
}

\newcommand{\util}{U}

\newcommand{\agent}{\mathrm{a}}
\newcommand{\setAgents}{\mathfrak{A}} 

\newcommand{\setPlayer}{\mathcal{P}} 
\newcommand{\player}{p}              
\newcommand{\carP}{\mathbf{p}}

\newcommand{\setArms}{\mathcal{A}}
\newcommand{\arm}{a}              
\newcommand{\strategy}{x}
\newcommand{\carA}{\mathbf{a}}

\newcommand{\rowStrategy}{x}

\newcommand{\rowAction}{i}
\newcommand{\rowSetAction}{\mathcal{I}}
\newcommand{\carI}{\mathbf{m}}

\newcommand{\columnStrategy}{y}

\newcommand{\columnAction}{j}
\newcommand{\columnSetAction}{\mathcal{J}}
\newcommand{\carJ}{\mathbf{k}}

\newcommand{\matching}{\mathfrak{m}}            
\newcommand{\Matching}{\mathbb{M}}
\newcommand{\pref}{\pi}              

\newtheorem{exmp}{Example}
\newtheorem{definition}{Definition}
\newtheorem{assumption}{Assumptions}

\newtheorem{theorem}{Theorem}
\newtheorem{Lemma}{Lemma}
\newtheorem{corollary}{Corollary}[theorem]
\newtheorem{proposition}{Proposition}
\newcommand{\proofstep}[1]{\noindent\textbullet \emph{#1.}}

\newenvironment{myproof}
{\par\noindent\textit{Proof. }}
{\hfill$\square$\par}

\DeclareMathOperator*{\argsort}{arg\,sort}
\DeclareMathOperator*{\argmin}{arg\,min}
\DeclareMathOperator*{\argmax}{arg\,max}

\begin{document}
\maketitle

\begin{abstract}
We introduce a learning problem in a generalized two-sided matching market, where agents select actions to interact with their match. Specifically, we consider a setting in which matched agents engage in zero-sum games with initially unknown payoff matrices, and we investigate whether a centralized procedure can learn an equilibrium from bandit feedback. We adopt the solution concept of a \emph{matching equilibrium}, where a matching \( \mathfrak{m} \) and a set of agent strategies \( X \) form an equilibrium if no agent has an incentive to deviate from \( (\mathfrak{m}, X) \). To quantify deviations of a candidate solution \( (\mathfrak{m}, X) \) from the equilibrium \( (\mathfrak{m}^\star, X^\star) \), we introduce the notion of \emph{matching instability}, which serves as a regret measure for the learning problem. We propose a UCB-based algorithm in which agents form preferences and select actions according to optimistic estimates of the payoffs. Our analysis establishes a sublinear, instance-independent regret upper bound, further supported by empirical evidence.
\end{abstract}


\section{Introduction}
Two-sided matching markets model situations in which two distinct sets of agents aim to match with each other based on individual preferences \cite{Roth_Sotomayor_1990}. Classic examples include students applying to universities, workers matching with firms, and organ donors paired with patients~\cite{roth1984evolution,gale_Shapley_1962,Haeringer}. A fundamental solution concept in this setting is \emph{stable matching}~\cite{gale_Shapley_1962}, a notion of equilibrium where no agent has an incentive to deviate from the proposed assignment. Stability has been widely studied in economics~\cite{gusfield1989stable}, while recent work in computer science considers online learning settings where agents’ preferences are initially unknown \cite{das2005two,pmlr-v108-liu20c,jagadeesan2021learning,athanasopoulos2025probably}.

Prior work on learning to match has mainly considered scenarios where, after being matched, agents directly receive a stochastic payoff that reflects their preferences, casting the problem in the Multi-Armed Bandit (MAB) framework \cite{das2005two}. However, a more realistic setting is one in which matched agents \emph{interact} through actions and obtain payoffs as outcomes of these interactions. For example, students may choose actions corresponding to the level of effort they invest, while universities may choose actions that adjust the academic environment, effectively engaging in a game (see also Appendix~\ref{app:mot} for further discussion).

We study this generalized \emph{learning scenario} in two-sided markets where, after being matched, agents engage in a two-player game that determines their utilities. 

\subsection{Related work}
\label{sec:RelatedWork}
\textbf{Social Games.} In a \emph{non-learning setting}  \cite{jackson2010social} introduced \emph{social games}, combining matching and games. They proposed the solution concept of a \emph{matching equilibrium}, where agents must play a Nash equilibrium (NE) with their match, and the matching \( \matching \) must be stable. Intuitively, this ensures that no agent pair can improve their utility by deviating from the match or strategy played. Below, we illustrate an example with a \emph{known} underlying zero-sum game (ZSG).
\begin{exmp}
Here there is one agent, $\{\player\}$, on the left side and two agents, $\{\arm_1, \arm_2\}$, on the right side, where the utilities of the actions in the corresponding ZSG are given by:

\begin{equation*}
\displaystyle
A_{\player, \arm_1} =
\begin{bmatrix}
1 & -1 \\
-1 & 1
\end{bmatrix}
\hspace{0.3in}
A_{\player, \arm_2} =
\begin{bmatrix}
1 & 1 \\
1 & 0
\end{bmatrix}
\end{equation*}
In this case, agents can naturally form their preferences based on the values of the NE in the respective ZSGs. Specifically, for the pair $(\player, \arm_1)$, the game value is $V^{\star}_{\player,\arm_1} = 0$, with the NE $(x_{p,\arm_1}^\star, x_{\arm_1}^\star)$, while for the second pair $(\player, \arm_2)$, the value is $V^{\star}_{\player,\arm_2} = 1$, with NE $(x_{\player, \arm_2}^\star, x_{\arm_2}^\star)$.

Thus, agent $p$ prefers $\arm_2$ over $\arm_1$, and the solution consisting of the matching $\matching = \{(\player, \arm_2)\}$ and the strategies $X = \{x_{\player,\arm_2}^\star, x_{\arm_2}^\star\}$ form a \emph{matching equilibrium} in the market. On the other hand, any deviation---either in the matching $\matching$ or in the strategies $X$--- can lead to an unstable outcome for the specific example \footnote{In general there is a non-empty set of stable matchings.}.
\end{exmp}

As illustrated, in the case of \emph{known} utilities, a stable outcome can be easily computed via an extension of the Gale-Shapley (GS) algorithm~\cite{jackson2008equilibrium}, which derives preferences based on the values of the game. In practice, however, agents may be \emph{unaware} of the underlying payoffs of the games (and thus their preferences), and instead must learn through repeated interactions, making the problem significantly more challenging. Our work is the first to consider this learning problem: we define assumptions and provide algorithms under which stable outcomes can be learned through interaction.

\textbf{Learning in two-sided matching.} The foundational work of \cite{das2005two} empirically evaluates multi-armed bandit (MAB) algorithms under specific preference settings. A follow-up by \cite{pmlr-v108-liu20c} introduced the notions of \emph{player-optimal} and \emph{player-pessimal} stable regret. Focusing on a setting where only one side of the market has unknown preferences, they analyze an explore-then-commit (ETC) algorithm that first explores to learn the preferences and then runs the GS algorithm on the estimated preference profile. They prove sublinear bounds for both regret notions. They also examine an upper confidence bound (UCB) approach, where agents maintain confidence estimates and repeatedly apply the GS algorithm using the preference ordering induced by the UCB estimates. For this approach, they establish sublinear bounds for player-pessimal stable regret; however, they prove an impossibility result for bounding player-optimal regret, demonstrating a fundamental challenge of the learning problem.

Later studies closed this gap by achieving sublinear player-optimal stable regret \cite{sankararaman2021dominate,basu2021beyond,kong2023player}, though they primarily focused on one-sided uncertainty and specific market structures. More recent studies have begun to consider settings where both sides of the market have uncertain preferences \cite{pagare2024explore,pmlr-v244-zhang24b}. However, these works employ ETC type algorithms, which are naturally expected to succeed given sufficient exploration.

A notable exception to the ETC framework is the work of \cite{jagadeesan2021learning,jagadeesan2021learning2}, which considers a UCB-based approach similar to \cite{pmlr-v108-liu20c} in settings with two-sided uncertainty. Their main model includes monetary transfers, though they also report results for a version without transfers as an ablation study. To establish sublinear regret in this setting, they propose an alternative regret notion of \emph{subset instability}, defined as the minimum subsidies required to stabilize the market. In our work, we extend this notion of regret to a setting where matched agents engage in a game, providing a suitable regret metric that captures the deviation from a matching equilibrium. Related concepts to subset instability have been explored in coalitional games, notably \emph{Cost of Stability}~\cite{bachrach2009cost} used to characterize games with an empty core, as well as the notion of the \(\epsilon\)-core~\cite{shapley1966quasi}.

\textbf{Learning in games.} There is a substantial body of research on learning in games studying uncoupled learning dynamics, where agents interact in a game without knowing the underlying game matrix or explicitly observing the opponents' actions~\cite{Nisan_Roughgarden_Tardos_Vazirani_2007,Cesa-Bianchi_Lugosi_2006}. This line of work builds on the development of no-regret learning algorithms \cite{hannan1957approximation,Cesa-Bianchi_Lugosi_2006}. In particular, in a two-player zero-sum game, if a single player employs a no-regret algorithm, their asymptotic average payoff is guaranteed to be at least the minimax value of the game, irrespective of the opponent's strategy. A stronger result arises in self-play: if both players employ no-regret algorithms, the joint empirical distribution of play asymptotically converges to a Nash equilibrium. In general-sum games, standard no-regret learning no longer guarantees convergence to a Nash equilibrium; instead, the empirical distribution of play converges to a coarse correlated equilibrium \cite{daskalakis2021near}. Subsequent work studies improved regret bounds \cite{daskalakis2011near,daskalakis2021near} and stronger guarantees for last-iterate convergence to equilibria \cite{daskalakis2017training}.

Since convergence to equilibrium is not easily guaranteed in the general case, we focus on zero-sum games, and by extension, constant-sum games. Moreover, integrating no-regret algorithms into our joint setting presents significant challenges, particularly in how agents form their preferences and in analyzing convergence to equilibrium when these preferences interact with matching mechanisms such as the GS algorithm. More specifically, our research is closely related to the work of \cite{o2021matrix}, who study learning in ZSGs under bandit feedback where agents additionally observe their opponents' actions. This information allows agents to form estimates of the underlying game matrix $A$ and apply optimism-based exploration strategies, leading to UCB-style algorithms. In self-play scenarios, their UCB approach empirically outperforms no-regret methods such as Exp3 \cite{auer1995gambling}, achieving faster convergence to the equilibrium. Crucially, their work evaluates performance using Nash regret, which compares the obtained rewards to the value of the game, rather than standard external regret. We adopt a similar information structure—observing the opponents' actions alongside bandit feedback—and propose a regret measure that reduces to Nash regret in the special case of ZSGs.

\subsection{Contributions}
We study an \emph{online learning} setting where agents interact in a two-player game with their assigned match, under \emph{unknown} payoff matrices with \emph{bandit} feedback. This creates a challenging learning scenario, as agents must simultaneously learn the underlying game and derive preferences from the consequences of their actions. To initiate the study of this joint setting, we introduce a set of assumptions and define an appropriate notion of regret that enables the theoretical analysis. In particular, our main contributions are:

\paragraph{1) Learning Problem:} We introduce a novel learning problem over a horizon \(T\), where at each step a central platform matches agents, and the matched agents engage in a ZSG for mathematical tractability.\footnote{Since ZSGs are equivalent to constant-sum games, our analysis can probably be extended easily to that case.} In our model, agents choose strategies conditional on their match, and after each interaction they observe a stochastic reward together with the action of their partner. This allows agents to maintain utility estimates over actions and adapt their strategies toward a stable outcome.

\paragraph{2) Regret formulation:}To characterize algorithms in our learning scenario, we propose an objective that measures the distance of the solution from equilibrium at each step. Specifically, we define \emph{matching instability} (Definition~\ref{def:m_instability}), which can be interpreted as the minimum subsidy required to stabilize the market. Our notion extends the \emph{subset instability} (Definition~\ref{def:s_instability}) introduced by~\cite{jagadeesan2021learning,jagadeesan2021learning2} for the two-sided matching problem. We also establish several properties of our metric that justify cumulative \emph{matching instability} as a suitable measure of regret.

\paragraph{3) Algorithm:} On the algorithmic side, we show that if agents follow the principle of optimism in the face of uncertainty, the learning process leads to sublinear regret. Specifically, we propose a UCB-like algorithm \cite{auer2002using} in which each agent maintains upper confidence bounds for the payoff matrices of the relevant games. Based on these optimistic estimates, agents form their preferences according to the estimated game values, which are then used to compute a stable matching via the GS algorithm. Once assigned, each agent selects a minimax strategy for their match, samples an action, and receives stochastic feedback.

\paragraph{4) Theoretical Analysis:} We prove that our algorithm achieves a regret of \( \tilde{O}\!\left(\sqrt{T\carI\carJ\carP\carA}\right) \), where \( \carI \) and \( \carJ \) denote the number of actions available to agents on the two sides, and \( \carP \) and \( \carA \) denote the number of agents on each side, respectively. This bound unifies and generalizes existing results in the literature, which are based on UCB-style algorithms. Specifically, when each agent has only one available action (i.e., \( \carI = \carJ = 1 \)), the setting reduces to a classical two-sided matching problem, and our regret matches the bound \( \tilde{O}(\sqrt{T\carP\carA}) \) from Theorem~6.6 in~\cite{jagadeesan2021learning2}. On the other extreme, when there is only one agent on each side (i.e., \( \carP = \carA = 1 \)), our regret becomes \( \tilde{O}(\sqrt{T\carI\carJ}) \). As we prove in Appendix~\ref{app:ZSG}, this matches the regret bound of the UCB algorithm proposed in \cite{o2021matrix} for ZSGs when both agents follow the algorithm in a self-play scenario.

\paragraph{5) Empirical Analysis:}We additionally conduct experiments on synthetic instances to further support our theoretical results. We compare our UCB algorithms in three different scenarios to gain better insights into the convergence of our approach. Specifically, we consider: 1) \textbf{Self-play}, where agents on both sides follow the UCB algorithm as described in our method; 2) \textbf{Nash-response}, where one side of the market selects the minimax strategy assuming known utilities for the agents; and 3) \textbf{Best-response}, where one side of the market commits to the best response, knowing the actions of the other opponent and their utilities.

\paragraph{6) Future directions:} As a broader contribution, our work opens several interesting research directions, including relaxing our current assumptions, exploring alternative models, and establishing stronger convergence guarantees to equilibria.

\paragraph{Paper structure.} Section~\ref{sec:prem} presents preliminaries on ZSGs and two-sided matching. Section~\ref{sec:social_games} describes the intersection of these settings, defines the learning problem, and introduces our proposed regret measure. Section~\ref{sec:Algo} presents our UCB algorithm, while Section~\ref{sec:Result} provides our main theorem alongside empirical results. Section~\ref{sec:conclusion} concludes and outlines potential future directions. Finally, all omitted proofs can be found in the Appendix.

\section{Preliminaries and Notation}
\label{sec:prem}
In this section, we formally introduce the concepts of two-player zero-sum games and the two-sided matching.

\textbf{Notation.} Throughout this paper, we use $\tilde{O}(\cdot)$ to hide logarithmic factors; formally, $f(x) = \tilde{O}(g(x))$ means that there exists a positive integer $k$ such that $f(x) = O(g(x) \log^k g(x))$. We use $1 \vee k := \max(1, k)$ for any integer $k$. For a positive integer $k$, we let $[n] := \{1, 2, \ldots, k\}$. We also write $\Delta_k$ to denote the probability simplex of size $k$. We also consider an element \( \agent \in \mathcal{S} = \mathcal{S}_1 \cup \mathcal{S}_2 \), where \( \mathcal{S}_1 \) and \( \mathcal{S}_2 \) are disjoint sets. We  define \( OS(\agent) \) as the opposite set of element \( \agent \), i.e., \( OS(\agent) = \mathcal{S}_2 \) if \( \agent \in \mathcal{S}_1 \), and \( OS(\agent) = \mathcal{S}_1 \) if \( \agent \in \mathcal{S}_2 \).

\subsection{Two-player zero-sum games}
\label{sec:zsg}
A two-player zero-sum game models an interaction between two competing agents, \( \player \) and \( \arm \), with \( \carI \) and \( \carJ \) actions, respectively. The game is specified by a pair of payoff matrices $(A_{\player}, A_{\arm})$ such that $A_{\player} = -A_{\arm}^T = A$, where $A \in [-1, 1]^{\carI \times \carJ}$ representing the utilities.

The agents select actions simultaneously: $\player$ chooses $\rowAction \in [\carI]$ and $\arm$ chooses $\columnAction \in [\carJ]$, after which they receive payoffs $A(i,j)$ and $-A(i,j)$, respectively. Agents typically select their actions from mixed strategies $\rowStrategy \in \Delta_m$ and $\columnStrategy \in \Delta_n$ i.e., probability distributions over their actions. The expected utilities are given by $\util_{\player}(\rowStrategy, \columnStrategy) = \rowStrategy^T A \columnStrategy$ and $\util_{\arm}(\columnStrategy,\rowStrategy) = -\columnStrategy^T A^T \rowStrategy$.

A central solution concept in game theory is the \emph{Nash equilibrium} (NE)~\cite{nash2024non}, in which no agent can unilaterally improve their utility by deviating from their equilibrium strategies $(\rowStrategy^{\star}, \columnStrategy^{\star})$ i.e., $ \util_p(\rowStrategy^{\star}, \columnStrategy^{\star}) \geq \util_p(\rowStrategy, \columnStrategy^{\star})$
and $\util_a( \columnStrategy^{\star},\rowStrategy^{\star}) \geq \util_a( \columnStrategy,\rowStrategy^{\star}) \forall \, \columnStrategy, \rowStrategy \in \Delta_m \times \Delta_n$.

The celebrated Minimax theorem~\cite{v1928theorie} states that the NE $(\rowStrategy^\star, \columnStrategy^\star)$ is attained if and only if each agent selects a strategy\footnote{The strategies can be computed efficiently via dual linear programs \cite{von1944morgenstern}} that maximizes their payoff against the opponent’s worst-case strategy, i.e., $\rowStrategy^{\star} = \argmax_{\rowStrategy} \min_{\columnStrategy}  \rowStrategy ^TA \columnStrategy$ and $\columnStrategy^{\star} = \argmin_{\columnStrategy} \max_{\rowStrategy } \columnStrategy^TA^T\rowStrategy$. The value of the game for each agent is then $V_{\player}^\star = \util_{\player}(\rowStrategy^{\star}, \columnStrategy^{\star})$ and $V_{\arm}^\star = \util_{\arm}(\columnStrategy^{\star}, \rowStrategy^{\star})$, with $V_{\player}^\star = -V_{\arm}^\star$.

\subsection{Two-sided matching}
\label{sec:matching}
The two-sided bipartite matching problem consists of two \emph{nonempty} finite sets of agents, denoted by $\setPlayer$ and $\setArms$, that aims to be matched one-to-one according to a matching $\Matching \subseteq \setPlayer \times \setArms$, representing a set of matched agents that are pairwise disjoint. Let $\setAgents = \setPlayer \cup \setArms$ denote the complete set of agents. By a slight abuse of notation, we also use the equivalent functional representation $\matching: \setAgents \rightarrow \setAgents \cup \{\bot\}$ of a matching, where $\matching(\player) = \arm$ and $\matching(\arm) = \player$ for the pair $(\player, \arm) \in \Matching$, and $\matching(\agent) = \bot$ if $\agent \in \setAgents$ is unmatched.

Each agent $\agent \in \setAgents$ maintains a complete preference list $\pi_{\agent}$ over agents on the other side, which is derived from a utility function. We denote the global utility function as $\util : \setAgents \times \setAgents \rightarrow [-1, 1]$, where $\util_{\agent,\matching(\agent)}$ represents the utility of agent $\agent$ for being matched with agent $\matching(\agent)$. Negative utilities are allowed to capture undesirable matches. The utilities of unmatched agents are arbitrary, i.e., $\util_{\agent,\bot} \in [-1, 1]$, whereas prior work typically assumes zero utility\footnote{We adopt this generalization to unify the notation required in our model.}. 

\textbf{Stability:} To characterize a matching $\matching$ that aligns with agents' preferences, \cite{gale_Shapley_1962} introduced stability as a notion of \emph{equilibrium} in the market, where no pair of agents has an incentive to deviate from $\matching$. A matching $\matching$ is stable if it is 1) \emph{individually rational}: each agent prefers their assigned match over remaining unmatched, i.e., 
$\util_{\agent, \matching(\agent)} \geq \util_{\agent, \bot}$ for all $\agent \in \setAgents$, 
and 2) \emph{has no blocking pairs}: there is no $(\player, \arm) \in \setPlayer \times \setArms$ such that 
$\util_{\player,\arm} > \util_{\player, \matching(\player)}$ and $\util_{\arm, \player} > \util_{\arm, \matching(\arm)}$.

They prove that the set of stable matchings $\mathcal{S}$ is always non-empty. They introduce the Gale--Shapley (GS) algorithm, in which agents on one side of the market sequentially propose to agents on the other side, while the side receiving the proposals is temporarily matched with their most preferred agents so far, until all agents are matched. The stable matching $\matching^{\star}_s$ produced by the algorithm is optimal for the proposing side and pessimal for the side receiving the proposals, in the sense that each agent is matched with their most/least preferred partner among all stable matchings in $\mathcal{S}$.

To measure the distance of a matching $\matching$ from equilibrium, \cite{jagadeesan2021learning} propose \emph{Subset Instability} (SI), which they also suggest as a suitable objective for learning stable matchings. The measure corresponds to the minimum subsidies $s_\agent$ that we have to provide to agents $\agent \in \setAgents$ to stabilize the market, i.e., ensuring the absence of blocking pairs and individual rationality (Constraints~\ref{si_c1} and~\ref{si_c2} in Definition \ref{def:s_instability}).

\begin{definition}[Subset Instability\footnote{Definition 6.4 from \cite{jagadeesan2021learning2}}]
\label{def:s_instability}
The subset instability $\mathbf{SI}(\matching, \util, \setAgents)$  of a matching \( \matching \) according to a utility function $\util:\setAgents \times \setAgents \rightarrow \mathbb{R}$ over a set of agents $\setAgents$ is defined by:
\begin{equation*}
    \min_{s \in \mathbb{R}} \sum_{\agent \in \setAgents} s_\agent
\end{equation*}
subject to:
\begin{align}
&\min\Big(
\util_{\player,\arm} - \util_{\player,\matching(\player)} - s_{\player},  \;\util_{\arm,\player} - \util_{\arm, \matching(\arm)} - s_{\arm}
\Big) \leq 0, 
\; \forall(\player, \arm) \in \setPlayer \times \setArms, \label{si_c1} \\
&\util_{\agent, m(\agent)} - \util_{\agent, \bot} + s_{\agent} \geq 0, \quad \forall \agent \in \setAgents, \label{si_c2} \\
&s_{\agent} \geq 0, \quad \forall \agent \in \setAgents. \label{si_c3}
\end{align}
\end{definition}

\section{Problem Setting}
\label{sec:social_games}
We introduce an \emph{online learning problem} in a generalized market, where matched agents engage in a game and receive \emph{bandit feedback} from their interactions.

\subsection{Learning Problem}
Specifically, we study a repeated interaction scenario over a horizon $T$, where at each step a central platform produces a matching $\matching_t$. The agents are paired and engage in a ZSG that depends on their match. They choose actions and receive stochastic rewards that reflect the utility of their actions.

We consider the special case where an agent $\agent$ is matched to \( \matching(\agent) \), and they play a \emph{distinct} \textbf{ZSG} described by the pair of payoff matrices $(A_{\agent,\matching(\agent)}, A_{\matching(\agent),\agent})$, with entities in range $[-1,1]$. The agents $\agent \in \setAgents$ then choose their strategies $\strategy_{\agent,\matching(\agent)}\in \Delta_{k_{\agent, \matching(\agent)}}$ for their match, with $k_{\agent, \matching(\agent)}$ represents the number of actions in their respective game. When the match is clear from the context we refer to the strategy of the agent $\agent$ with $\strategy_{\agent}$. We also denote with $X=\{\strategy_{\agent}\}_{\agent \in \setAgents}$ the set of strategies for the agents.

The expected utility of agent $\agent$ when matched with $\matching(\agent) $, using strategies $ (\strategy_{\agent}, \strategy_{\matching(\agent)}) $, is given by:
\begin{equation}
    \util_{\agent,\matching(\agent)}(\strategy_{\agent}, \strategy_{\matching(\agent)}) = \strategy_{\agent}^T A_{\agent,\matching(\agent)} \strategy_{\matching(\agent)},
  \end{equation}

We now describe the learning process, where at each step $t = 1, \ldots, T $:
\begin{enumerate}
    \item A central platform chooses a matching $\matching_t$, based on the estimated agent preferences $\hat{\pi}_{\agent}$.
    \item Each agent $\agent \in \setAgents$ observes their match $\matching_t(\agent)$, and selects a strategy $\strategy_{\agent}^t$.
    \item The agents $\agent \in \setAgents$ draw their actions \( i_{\agent}^t \sim \strategy_{\agent}^t \) from their strategies.
    \item Each agent \( \agent \in \setAgents \) receives a random reward \( r_{\agent,\matching_t(\agent)}^t \) with expectation $A_{\agent,\matching_t(\agent)}( i_{\agent}^t, i_{\matching_t(\agent)}^t)$.
    \item The agents $\agent \in \setAgents$ observes the action of their match, $i_{\matching_t(\agent)}^t$, and updates their estimate of the reward $\hat{A}^t_{\agent,\matching_t(\agent)}(i_{\agent}^t, i_{\matching_t(\agent)}^t)$.
    \item The agents update their preferences $\hat{\pi}_{\agent}$ based on the estimates $\hat{A}^t_{\agent,\matching_t(\agent)}$.
\end{enumerate}

At a high level, after the matching $\matching_t$, the agents simultaneously select actions to interact after observing their match. They then observe a stochastic reward and the actions of their match. This information allows them to estimate the payoffs and form the preferences that they report back to the platform, which in turn determines the matching for the next round. In Section~\ref{sec:Algo}, we present the full algorithm for our setting.

\begin{assumption}
\label{assumption}
Below, we enumerate the list the assumptions for our learning problem.
\begin{enumerate}[label=A\arabic*., leftmargin=*]
    \item We assume the utilities for being unmatched \( \util_{\agent, \bot} \in [-1,1] \)  are known  \( \forall \agent \in \setAgents \). \label{assumption:A1}
    \item The agents truthfully report their preferences to the central platform. \label{assumption:A4}
    \item To simplify notation, we assume that agents on the same side have the same number of available actions for each possible match, i.e., $A_{\player,\arm} \in [-1, 1]^{\carI \times \carJ} \forall (\player, \arm) \in \setPlayer \times \setArms$. \label{assumption:A2}
    \item We assume that the stochastic reward $r_{\agent,\matching(\agent)}^t$ is a 1-sub-Gaussian random variable with expected value $\mathbb{E}_{\agent,\matching(\agent)}[r_{\agent,\matching(\agent)}^t \mid i_{\agent}^t, i_{\matching(\agent)}^t] = A_{\agent,\matching(\agent)}(i_{\agent}^t, i_{\matching(\agent)}^t)$. \label{assumption:A3}
\end{enumerate}
\end{assumption}
Assumption~\ref{assumption:A1} is standard, where utilities are often set to zero. We instead assume arbitrary utilities in $[-1, 1]$ to allow for non-trivial individually rational matchings\footnote{In a ZSG, one agent’s gain is the other’s loss, i.e., $\util_{\agent, \matching(\agent)} = -\util_{\matching(\agent), \agent}$. Setting $\util_{\agent, \bot} = 0$ may violate individual rationality, i.e., $\util_{\agent, \matching(\agent)} \geq \util_{\agent, \bot}$, leading to trivial empty matchings.}. Assumption~\ref{assumption:A4} is standard in centralized matching  \cite{liu2021bandit}\footnote{\cite{liu2021bandit} also study whether an agent can benefit from misreporting their preferences, showing a positive result in the case of a single stable matching, while for the general setting the problem remains open.} and can be enforced by having the mechanism observe the rewards and implement the algorithm. Assumption~\ref{assumption:A2} is made for notational convenience and is not restrictive, as the number of actions can always be bounded by the maximum across all agent pairs. Finally, Assumption~\ref{assumption:A3} requires the rewards to be 1-sub-Gaussian for simplicity, although our results can also be generalized to bounded rewards using Hoeffding's inequality.


\subsection{Equilibrium}
Matching with strategic agents was initially proposed by \cite{jackson2008equilibrium} as \emph{social games} in a \emph{non-learning} setting. They propose a \emph{matching equilibrium}, as a notion of stability, where intuitively no pair of agents has an incentive to deviate from the matching and instead play a Nash equilibrium with each other.

In order for the solution $(\matching, X)$ to be stable, three conditions must hold: (i) agents must be individually rational: they should prefer their assigned match over remaining unmatched; (ii) agents must play a Nash equilibrium with their match; and (iii) there should be no blocking pair, i.e. no unmatched pair of agents who would both prefer to be matched with each other under the resulting equilibrium strategies. Below, we formally define \emph{matching equilibrium} for cardinal utilities:

\begin{definition}[Matching Equilibrium]
\label{def:matching_equilibria}
A \emph{matching equilibrium}, is a pair $(\matching, X)$, consisting of a matching $\matching$ and a profile of strategies $X = \{x_{a}\}_{a \in \setAgents}$, such that the outcome satisfies the following conditions:\begin{enumerate}
    \item \textbf{Individual rationality:}  $\util_{\agent, \bot} \leq  \util_{\agent, m(\agent)}(x_{\agent},x_{\matching(\agent)})$ for all $\agent \in \setAgents$. \label{item:IR}
    \item \textbf{Nash Equilibrium:} for all $\agent \in \setAgents$ \\ $U_{\agent, \matching(\agent)}(x'_{\agent},x_{\matching(\agent)}) \leq \util_{\agent, \matching(\agent)}(x_{\agent},x_{\matching(\agent)})$ $\, \forall x'_{\agent} \in \Delta_x$ \label{item:Nash}\item \textbf{No blocking pairs:} There exists no pair $(\player,\arm) \in \setPlayer \times \setArms$ such that:
    $$
    V^{\star}_{\player,\arm} > \util_{\player, \matching(\player)}(x_{\player},x_{\matching(\player)})
    $$
    and
    $$V^{\star}_{\arm,\player} > \util_{\arm, \matching(\arm)}(x_{\arm}, x_{m(\arm)}). \label{item:Blocking}
    $$
\end{enumerate}
\end{definition}


\cite{jackson2010social} prove that \emph{matching equilibrium} always exist, by extending the classical result of \cite{gale_Shapley_1962}. To see this, consider running the GS algorithm, where agents naturally form preferences over potential partners based on the value of the induced games at their NE. Specifically, assuming that agents know the payoffs of the games, each agent $\agent$ can rank the agents on the opposite side $\agent'$ according to the value $V^*_{\agent,\agent'}$ of the NE $(\strategy^*_{\agent}, \strategy^*_{\agent'})$ in their corresponding ZSG. Note that in two-player ZSG, a NE always exists and all equilibria yield the same value $V^*$.

\paragraph{Challenges:} Our problem is to efficiently learn an equilibrium in stable matching games. In matching problems, algorithms are typically based on GS, with agent preferences derived from upper confidence estimates \cite{liu2021bandit,jagadeesan2021learning}. In contrast, learning in games often relies on no-regret adversarial algorithms, such as MWU or EXP3 \cite{MAL-068}. Combining these approaches poses challenges regarding how agents should form their preferences using no-regret algorithms, how to analyze the behavior of the GS algorithm. In the following section, we introduce a weaker notion of distance to equilibrium that facilitates learning in our setting.


\subsection{Regret}
We now introduce our incentive-aware learning objective, designed to analyze the cumulative behavior of learning algorithms. We define the notion of \emph{matching instability} to quantify the distance between a solution \( (\matching, X) \) and the equilibrium \( (\matching^{\star}, X^{\star}) \) at each time step, extending the \emph{subset instability} metric
(Definition~\ref{def:s_instability}) originally proposed
by~\cite{jagadeesan2021learning,jagadeesan2021learning2} for two-sided matching markets.

\begin{definition}[Matching Instability]
\label{def:m_instability}
The matching instability $\mathbf{MI}(m, X)$ of a matching \( \matching \) under mixed strategies $X = \{\strategy_{\agent}\}_{\agent \in \setAgents}$ is defined as:
\begin{equation*}
    \min_{s \in \mathbb{R}} \sum_{a \in \setAgents} s_a
\end{equation*}
subject to:
\noindent
\begin{align}
&\min(V^{\star}_{\player,\arm} - \util_{\player,\matching(\player)}(\strategy_{\player},\strategy_{\matching(\player)}) - s_{\player}, 
V^{\star}_{\arm,\player} - \util_{\arm,\matching(\arm)}(\strategy_{\arm},\strategy_{\matching(\arm)}) - s_{\arm}) \leq 0, \forall_{\player,\arm \in\setPlayer\times\setArms} \label{eq:C11}\tag{C1}\\
&\util_{\agent,\matching(\agent)}(\strategy_{\agent},\strategy_{\matching(\agent)}) - \util_{\agent,\bot} + s_{\agent} \geq 0, \; \forall \agent \in \setAgents \label{eq:C22}\tag{C2}\\
&V^{\star}_{\agent,\matching(\agent)} - \util_{\agent,\matching(\agent)}(\strategy_{\agent},\strategy_{m(\agent)}) - s_{\agent} \leq 0, \; \forall \agent \in \setAgents \label{eq:C33}\tag{C3}\\
&s_{\agent} \geq 0, \; \forall \agent \in \setAgents \label{eq:C44}\tag{C4}
\end{align}
\end{definition}

Our metric can be interpreted as the minimum subsidies that the platform must provide to users either to retain them on the platform or, equivalently, to achieve stability by satisfying the constraints of the \emph{matching equilibrium} in Definition~\ref{def:matching_equilibria}. Specifically: (i) constraint~\ref{eq:C11} enforces the absence of blocking pairs under Nash strategies, (ii) constraint~\ref{eq:C22} ensures individual rationality, and (iii) constraint~\ref{eq:C33} guarantees that the utilities of the agents are at least as high as the minmax value of the game.

Constraints~\ref{eq:C11} and~\ref{eq:C22} align exactly with the requirements of matching equilibrium. Constraint~\ref{eq:C33} relaxes the Nash equilibrium condition, requiring instead that agents obtain a utility no less than the equilibrium value. This choice enables a tractable theoretical analysis of regret circumvents the challenges of incorporating adversarial no-regret algorithms, which are typically employed to converge to NE. Our perspective is consistent with prior work that evaluates algorithmic performance relative to the value of the game \cite{jrob,o2021matrix,daskalakis2011near,Cesa-Bianchi_Lugosi_2006}. Next, we describe properties of our instability measure important for learning.

\begin{proposition}
\label{prop:1}
The Matching Instability \( \mathbf{MI}(\matching, X) \) of a solution \( (\matching, X) \) satisfies the following properties:
\begin{enumerate}
    \item \( \mathbf{MI}(\matching, X) \) is always non-negative.
   \item \( \mathbf{MI}(\matching, X) \) equal to zero if and only if:
    \begin{itemize}
    \item[A)] Each agent receives the same valuation as in the NE of their respective game, i.e., 
    \( \util_{\agent,\matching(\agent)}(x_{\agent},x_{\matching(\agent)}) = V^{\star}_{\agent,\matching(\agent)} \forall \agent \in \setAgents \).
    \item[B)] Matching \( \matching \) is stable under the strategies $X$.
    \end{itemize}
    \item In the case where each agent has a single action, i.e., \( \carI = \carJ = 1 \), the model reduces to the two-sided matching problem of Section \ref{sec:matching}, and the matching instability coincides with the subset instability of Definition~\ref{def:s_instability}.
    \item In the case of a single agent on each side, i.e., \( \setPlayer = \{ \player \} \) and \( \setArms = \{ \arm \} \), where they are always matched, i.e., \( m(\player) = \arm \), the model reduces to the ZSG setting. In this case, the matching instability \( \mathbf{MI}(\matching, X) \) for strategies \( X = \{ \strategy_\player, \strategy_\arm \} \) becomes:
    \begin{equation*}
        \mathbf{MI}(\matching, X) = \left| V^{\star}_{\agent,\player} - U_{\agent, \player}(\strategy_\agent, \strategy_{\player}) \right|,
    \end{equation*}
    which measures the absolute difference from the value of the game. This is similar to the notion of absolute Nash regret, which is used as a proxy for the distance from a Nash equilibrium in ZSG \cite{o2021matrix}.
\end{enumerate}
\end{proposition}

\noindent\textbf{Regret:} We define the regret of an algorithm as the cumulative matching instability over a horizon $T$:
\begin{equation*}
    R_T = \sum_{t=1}^T \mathbf{MI}(\matching_t, X_t)
\end{equation*}

\section{UCB Algorithm}
\label{sec:Algo}
\begin{algorithm}
\caption{\textsc{UCB-MG}: UCB-Based Algorithm for Matching Games}
\label{alg:ucb_algo}
\begin{algorithmic}[1]
    \State \textbf{Input:} $T$, $\delta$, $\setPlayer$, $\setArms$, $\rowSetAction$, $\columnSetAction$,
    \State $n_{\player,\arm}(i,j) = 0 \; \forall \; (\player,\arm,i,j) \in \setPlayer\times\setArms\times\rowSetAction\times\columnSetAction$
    \State $\hat{A}^{t=0}_{\player,\arm} = 0 \; \forall \; (\player,\arm) \in \mathcal{P}\times\mathcal{A}$
    \For{$t \in 1,\cdots, T$}
        \State $\bar{A}_{\player,\arm}(i,j) = \hat{A}^t_{\player,\arm}(i,j) + \sqrt{\frac{2 \log(1/\delta)}{1 \vee n_{\player,\arm}(i,j) }} \; \forall_{\player, \arm, i, j}$
        \State $\bar{A}_{\arm,\player}(j,i) = -\hat{A}^t_{\player,\arm}(i,j) + \sqrt{\frac{2 \log(1/\delta)}{1 \vee n_{\player,\arm}(i,j) }} \; \forall_{\player, \arm, i, j}$
        \State $\bar{V}_{\player,\arm} = max_{x \in \Delta_x} min_{y \in \Delta_y} x^T \bar{A}_{\player,\arm} y \; \forall_{\player,\arm}$
        \State $\bar{V}_{\arm,\player} = max_{y \in \Delta_y} min_{x \in \Delta_x} y^T \bar{A}_{\arm,\player} x \; \forall_{\player,\arm}$
        \State $\hat{\pi}_{\agent} = \argsort_{i} \bar{V}_{\agent, i} \; \forall \agent \in \setAgents $
        \State $\matching_t = GS(\{\hat{\pi}\}_{\agent \in \setAgents})$
        \For{$(\player,\arm) \in \matching_t$}
            \State $x = arg \max_{x \in \Delta_x} min_{y \in \Delta_y} x^T \bar{A}_{\player,\arm} y$
            \State $y = arg \max_{y \in \Delta_y} min_{x \in \Delta_x} y^T \bar{A}_{\arm,\player} x$
            \State sample actions $i \sim x$ and $j \sim y$
            \State sample $r^t_{\player,\arm}(i,j)$ and update $\hat{A}_{\player,\arm}^t(i,j)$
            \State $n_{\player,\arm}(i,j) = n_{\player,\arm}(i,j) + 1 $
        \EndFor
    \EndFor
\end{algorithmic}
\end{algorithm}

We now discuss how agents follow the optimism-in-the-face-of-uncertainty principle from a self-play perspective. Since agents observe the actions of their matched partners after each interaction, they form their preferences and select their strategies based on upper confidence estimates of the payoff matrix. A central platform can thus match the agents \( \agent \in \setAgents \) according to preferences derived from these upper confidence estimates. Subsequently, each agent selects their min--max strategies \(X\) based on the upper confidence payoffs associated with their assignment. For clarity, Algorithm~\ref{alg:ucb_algo} outlines this self-play learning procedure as a single algorithm.

More formally, at each time step \( t \in [T] \),
each agent \( \agent \in \setAgents \) maintains an optimistic estimate of the payoff matrix \( \bar{A}^t_{\agent,\agent'} \) for playing a game with agent \( \agent' \in OS(\agent) \cup \{\bot\} \). Based on these estimates, agents form their preferences \( \bar{\pref}_{\agent} \) according to the game values \( \bar{V}^{\star}_{\agent,\agent'} \) when matched with agent \( \agent' \), i.e., \( \bar{\pref}_\agent = \argsort_{\agent'} \bar{V}^{\star}_{\agent,\agent'} \)\footnote{Note that the values of the game \( \bar{V}^{\star}_{\agent,\agent'} \) can be computed without knowledge of the opponent's strategy.}. Each agent then reports their preferences to the central platform, which produces a matching \( \matching_t \) using the GS algorithm based on the estimated preferences.

Once matched, each agent $\agent \in \setAgents$ observes their match
$\matching_t(\agent)$ and samples actions \( i_{\agent}^t \) from
their respective min-max strategies \( x^t_{\agent}\) according to the
upper confidence payoff matrix
\( \bar{A}^t_{\agent,\matching_t(\agent)}\).  After choosing their
actions, the agents \( \agent \in \setAgents \) observe the actions of
their matched partners \( \matching_t(\agent) \), and receive a reward
\( r_{\agent,\matching_t(\agent)}^t \), with
\( r_{\matching_t(\agent),\agent}^t =
-r_{\agent,\matching_t(\agent)}^t \). These observations are used to
update the empirical mean of the payoff matrix
$\hat{A}^t_{\agent, \matching(\agent)}(i,j)$.

The upper confidence bounds of the payoff matrices for each agent $\agent \in \setAgents$ can be computed as:
\begin{equation*}
    \bar{A}^t_{\agent, \agent'}(i,j) = \hat{A}^t_{\agent, \agent'}(i,j) + \sqrt{\frac{2 \log(1/\delta)}{1 \vee n_{\agent, \agent'}^t(i,j)}}, \forall \agent' \in OS(\agent) 
\end{equation*}
where \( n_{\agent, \agent'}^t(i,j) \) is the number of times agents \( (\agent, \agent') \) played actions \( (i, j) \) up to time \( t \).

\section{Results}
\label{sec:Result}
We now show that the self-play approach outlined in Algorithm~\ref{alg:ucb_algo} gradually guides the agents toward a matching equilibrium. In particular, our theorem states that the cumulative matching instability incurred over time grows sublinearly with the horizon \(T\). The proof is provided in the Appendix.

\begin{theorem}
\label{theo:ucb}
 Let $\mid \setPlayer \mid = \carP$, $\mid \setArms \mid = \carA$, $\mid \rowSetAction \mid = \carI$ and $\mid \columnSetAction \mid = \carJ$. When Algorithm~\ref{alg:ucb_algo} is run with parameter $\delta= 1/(4  T^2  \carP^2\carA^2\carI \carJ)$, the expected regret $R_T$ under horizon $T$ is bounded by:
 \begin{equation*}
    \mathbb{E}[R_T] = \mathbb{E}[\sum_{t=1}^T \mathbf{MI}(m_t, X_t)] \leq \tilde{O}\left(\sqrt{T\mathbf{k}\mathbf{m}\mathbf{p}\mathbf{a}} \right).
\end{equation*}
\end{theorem}
This regret bound recovers and generalizes several existing results in the literature. In particular, when each agent has access to only a single action (i.e., \( \carI = \carJ = 1 \)), the problem reduces to a classical two-sided matching scenario, and the bound simplifies to \( \tilde{O}(\sqrt{T \carP \carA}) \), similar to the bound of Theorem 6.6 in~\cite{jagadeesan2021learning2}. When there is only one agent on each side (i.e., \( \carP = \carA = 1 \)), the problem reduces to a ZSG and the regret becomes \( \tilde{O}(\sqrt{T \carI \carJ}) \). In this case, the algorithm corresponds to a direct application of the UCB algorithm of~\cite{o2021matrix} in a self-play setting, where we show in Appendix~\ref{app:ZSG} that yields the same regret bound.


\subsection{Simulations}
\label{sec:Experiment}
\begin{figure*}[t]
\centering
\includegraphics[width=\textwidth]{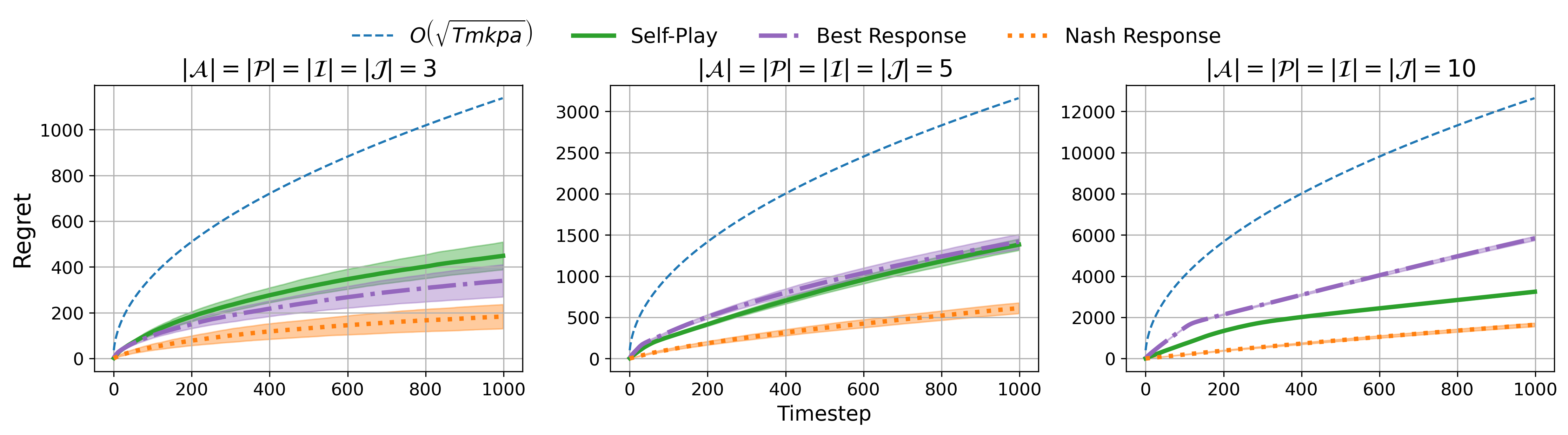}
\caption{Regret of the different settings, including the theoretical bound. Lines represent the average performance of the algorithm, and shaded areas indicate the standard deviation across runs.}
\label{fig:regret_self_play}
\end{figure*}
We now present a set of simulations that complement our theoretical results, in which we empirically compare three different scenarios. We generate 50 random instances by independently sampling the entities payoff matrices of the ZSG from a normal distribution, i.e., $A_{p,a}(\rowAction,\columnAction) \sim \mathcal{N}(0,1) \; \forall (\player, \arm, \rowAction, \columnAction) \in \setPlayer \times \setArms \times \rowSetAction \times \columnSetAction$. The code for reproducing all the experiments is provided in the supplementary material\footnote{\url{https://github.com/a-athanasopoulos/Learning-in-Matching-Games}} while additional details are provided in the Appendix.

We compare the \textbf{self-play} scenario outlined in Algorithm~\ref{alg:ucb_algo} with two additional cases in which agents on one side, $\setArms$, have access to additional information and can act differently, in order to study how the regret scales relative to these settings. First, we consider the \textbf{Nash-response} setting. The agents in $\setPlayer$ follow Algorithm~\ref{alg:ucb_algo}, while the agents in $\setArms$ form their preferences and select strategies according to the NE of the true payoff matrices \(A\). In this case, we expect the regret to be lower, as the agents on one side know the preferences and play according to the NE. Next, we consider the other extreme, the \textbf{Best-response} setting, where agents \( \arm \in \setArms \) know the true payoff matrices and additionally observe the strategies selected by agents in \( \setPlayer \), i.e., \( x_{\player,\arm} \) for all \( \player \in \setPlayer \), before selecting their own strategies. They then form their preferences and choose strategies that exploit their match by playing best responses. Specifically, each agent computes their best-response strategy as \( \strategy_{\arm,\player} = \argmax_{\strategy} \util_{\arm,\player}(\strategy, \strategy_{\player,\arm}) \), and form their preferences according to the resulting payoffs: \( \pref_{\arm} = \argsort_{\player \in \setPlayer} \util_{\arm,\player}(\strategy_{\arm,\player}, \strategy_{\player,\arm}) \).

Figure~\ref{fig:regret_self_play} compares the regret for the different settings for varying numbers of agents and actions, along with the theoretical bound of $O\left(\sqrt{T\carA\carP\carI\carJ}\right)$ (dashed blue), where we omit the logarithmic term for clarity. In all experimental settings, we observe sublinear regret growth, indicating that the system progressively learns a stable outcome. The \textbf{Nash-response} (dotted orange) scenario yields the lowest regret as excepted, as one side of the market has full knowledge of the payoff matrices and acts according to NE. In addition, the \textbf{Best-response} (dash-dotted purple) scenario initially outperforms \textbf{Self-play} (solid green), but as the number of agents and actions increases, its performance deteriorates. This is due to agents in \( \setArms \) exploiting their matched partners in \( \setPlayer \), which introduces greater complexity. Plots for different market sizes are provided in the Appendix.

\section{Conclusion and Future Work}
\label{sec:conclusion}
We introduce a novel problem of learning equilibria from bandit feedback in a generalized two-sided matching market, where interactions between matched agents follow a zero-sum game. We precisely formulate the learning protocol, specifying the required assumptions, and introduce a suitable measure of regret. Finally, we propose an algorithm based on the principle of optimism in the face of uncertainty, which achieves sublinear regret demonstrating that \emph{matching equilibria} can be learned efficiently in this setting.

The motivation for modeling matching with strategic agents stems from the observation that, in many social and economic environments, agents receive payoffs only after interacting through their actions. Our study represents an initial step toward understanding learning in this setting and opens several promising research directions, including relaxing our current assumptions, exploring alternative models, and establishing stronger notions of convergence to the equilibria.

The most direct direction for future research is to study \textbf{richer game-theoretic models}, moving beyond zero-sum games to general-sum games. A central challenge in this setting is how agents can form their preferences while following no-regret algorithms when interacting with their match. In this context, integrating no-regret algorithms with matching mechanisms such as the GS algorithm presents significant analytical challenges. Moreover, allowing agents to employ no-regret learning may enable the study of stronger interaction dynamics, relaxing the assumption that agents directly observe the actions of their match. The use of no-regret algorithms may also lead to stronger notions of convergence to equilibrium, since in our current setting the agents only converge toward the value of the game. Another direction is to explore \textbf{alternative game-theoretic models}, such as Markov games or Stackelberg games, which can better capture strategic behavior and temporal dynamics in one-to-one matching problems. \textbf{Alternative matching models} are also intresting, including many-to-one markets, potentially modeled using polymatrix games~\cite{cai2016zero}. Additional directions include learning equilibria in \textbf{adversarial environments}, such as adversarially changing games~\cite{pmlr-v97-cardoso19a}. Finally, addressing the \textbf{decentralized case}, where agents select both their match and strategy without central coordination, remains an important open research problem.

\bibliographystyle{unsrt}  
\bibliography{references}  

\appendix
\section*{Appendix Outline}
Now, we provide the detailed proofs and discussion omitted from the main text:
\begin{itemize}
    \item \textbf{Appendix~\ref{app:related_work}:} Provides an extended related work.
    \item \textbf{Appendix~\ref{app:mot}:} Further discussion on motivation examples.
    \item \textbf{Appendix~\ref{app:ZSG}:} Provides a corollary to the theorem of paper \cite{o2021matrix} for their algorithm in the self-play scenario.
    \item \textbf{Appendix~\ref{sec:app_prop}:} Contains the proof of Proposition~\ref{prop:1} from Section 3.
    \item \textbf{Appendix~\ref{sec:app_proof}:} Contains the proof of Theorem~\ref{theo:ucb} from Section 5.
    \item \textbf{Appendix~\ref{app:exp}:} Contains the additional information on the empirical study together with additional result omited in the main paper.
\end{itemize}

\section{Related work}
\label{app:related_work}
\textbf{Stable matching problem.} Two-sided matching markets model situations in which two distinct sets of agents aim to match with each other based on individual preferences~\cite{gusfield1989stable,Roth_Sotomayor_1990}. Classic examples include students applying to universities, workers matching with firms, and organ donors paired with patients~\cite{roth1984evolution}. In their fundamental work, \cite{gale_Shapley_1962} introduced the concept of stable matching, a notion of equilibrium in which no two agents have an incentive to deviate. \cite{gale_Shapley_1962} proved that the set of stable matchings $\mathcal{S}$ is always non-empty by introducing the Gale–Shapley (GS) algorithm, which is based on sequential proposals from one side of the market to the other. The algorithm returns the (unique) \emph{optimal stable matching} $m^\star$ for the proposing side, in the sense that every proposer receives the best partner they can obtain in any stable matching. Conversely, the resulting matching is pessimal for the side receiving the proposals, providing each agent on that side with their least-preferred partner among all stable matchings.

\textbf{Social Games.} In the \emph{non-learning setting}, the intersection of matching and games was introduced by \cite{jackson2010social} named as social games. They propose the solution concept of a \emph{matching equilibrium}, where agents must play a Nash equilibrium (NE) with their match, and the matching \( \matching \) must be stable. Intuitively, a pair of agents dissatisfied with their current outcome,  may have an incentive to deviate and form a new match in which they adopt a different NE. Below, we illustrate an example with \emph{known} underlying zero-sum games (ZSGs).

\begin{exmp}
Consider the following example with one agent, $\{\player\}$, on the left side and two agents, $\{\arm_1, \arm_2\}$, on the right side. For the sake of the argument, we assume that agents always prefer being matched rather than remaining single. The utilities of the actions in the respective zero-sum games are given by:

\begin{equation*}
\displaystyle
A_{\player, \arm_1} =
\begin{bmatrix}
1 & -1 \\
-1 & 1
\end{bmatrix}
\hspace{0.3in}
A_{\player, \arm_2} =
\begin{bmatrix}
1 & 1 \\
1 & 0
\end{bmatrix}
\end{equation*}
Agents can naturally form their preferences based on the values of the Nash equilibrium in the respective ZSGs. Specifically, for the pair $(\player, \arm_1)$, the game value is $V^{\star}_{\player,\arm_1} = 0$, with the Nash equiblirium $(x_{p,\arm_1}^\star, x_{\arm_1}^\star)$ given by $x_{p,\arm_1}^\star = [0.5,\ 0.5]$ and $x_{\arm_1}^\star = [0.5,\ 0.5]$, while for the second pair $(\player, \arm_2)$, the value is $V^{\star}_{\player,\arm_2} = 1$, with Nash strategies $x_{\player, \arm_2}^\star = [1,\ 0]$ and $x_{\arm_2}^\star = [0,\ 1]$.

Thus, agent $p$ prefers $\arm_2$ over $\arm_1$ under Nash strategies $(x_{\player, \arm_2}^\star,x_{\player, \arm_1}^\star)$, and the solution consisting of the matching $\matching = \{(\player, \arm_2)\}$ and the strategies $X = \{x_{\player,\arm_2}^\star, x_{\arm_2}^\star\}$ form a \emph{matching equilibrium} in the market.

On the other hand, any deviation---either in the matching $\matching$ or in the strategies $X$--- can lead to an unstable outcome for the specific example \footnote{In general there is a non-empty set of stable matchings.}.
\end{exmp}

As illustrated, in the case of \emph{known} utilities, a stable outcome can be easily computed via an extension of the Gale-Shapley (GS) algorithm~\cite{jackson2008equilibrium}, which derives preferences based on the values of the game. In practice, however, agents may be \emph{unaware} of the underlying payoffs of the games (and thus their preferences), and instead must learn through repeated interactions, making the problem significantly more challenging. In our work, we first consider this learning problem and make an initial contribution by specifying under which assumptions, and in what sense, such a stable outcome can be learned.

\textbf{Learning in two-sided matching.} The foundational work of \cite{das2005two} empirically evaluates multi-armed bandit (MAB) algorithms under specific preference settings. A follow-up by \cite{pmlr-v108-liu20c} introduced the notions of \emph{player-optimal} and \emph{player-pessimal} stable regret. Focusing on a setting where only one side of the market has unknown preferences, they analyze an explore-then-commit (ETC) algorithm that first explores to learn the preferences and then runs the GS algorithm on the estimated preference profile. They prove sublinear bounds for both regret notions. They also examine an upper confidence bound (UCB) approach, where agents maintain confidence estimates and repeatedly apply the GS algorithm using the preference ordering induced by the UCB estimates. For this approach, they establish sublinear bounds for player-pessimal stable regret; however, they prove an impossibility result for bounding player-optimal regret, demonstrating a fundamental challenge of the learning problem.

Later studies closed this gap by achieving sublinear player-optimal stable regret \cite{basu2021beyond,kong2023player}, though they primarily focused on one-sided uncertainty and specific market structures. More recent studies have begun to consider settings where both sides of the market have uncertain preferences \cite{pagare2024explore,pmlr-v244-zhang24b}. However, these works employ ETC type algorithms, which are naturally expected to succeed given sufficient exploration.

A notable exception to the ETC framework is the work of \cite{jagadeesan2021learning}, which considers a UCB-based approach similar to \cite{pmlr-v108-liu20c} in settings with two-sided uncertainty. Their main model includes monetary transfers, though they also report results for a version without transfers as an ablation study. To establish sublinear regret in this setting, they propose an alternative regret notion of \emph{subset instability}, defined as the minimum subsidies required to stabilize the market. In our work, we extend this notion of regret to a setting where matched agents engage in a game, providing a suitable regret metric that captures the deviation from a matching equilibrium. Related concepts to subset instability have been explored in coalitional games, notably \emph{Cost of Stability}~\cite{bachrach2009cost} used to characterize games with an empty core, as well as the notion of the \(\epsilon\)-core~\cite{shapley1966quasi}.

\textbf{Learning in games.} There is a substantial body of research on learning in games studying uncoupled learning dynamics, where agents interact in a game without knowing the underlying game matrix or explicitly observing the opponents' actions~\cite{Nisan_Roughgarden_Tardos_Vazirani_2007,Cesa-Bianchi_Lugosi_2006}. This line of work builds on the development of no-regret learning algorithms \cite{hannan1957approximation,Cesa-Bianchi_Lugosi_2006}. In particular, in a two-player zero-sum game, if a single player employs a no-regret algorithm, their asymptotic average payoff is guaranteed to be at least the minimax value of the game, irrespective of the opponent's strategy. A stronger result arises in self-play: if both players employ no-regret algorithms, the joint empirical distribution of play asymptotically converges to a Nash equilibrium. In general-sum games, standard no-regret learning no longer guarantees convergence to a Nash equilibrium; instead, the empirical distribution of play converges to a coarse correlated equilibrium \cite{daskalakis2021near}. A more refined result states that if agents has no \textit{internal} regret, their joint empirical distribution converges to the set of correlated equilibria \cite{Nisan_Roughgarden_Tardos_Vazirani_2007}. Subsequent research has focused on improving regret bounds \cite{daskalakis2021near} and establishing stronger guarantees for last-iterate convergence to equilibria, driven by its practical importance in modern machine learning applications, particularly in stabilizing the training dynamics of generative adversarial networks \cite{daskalakis2017training}.

In our work, we focus on ZSGs to reduce the additional complexity inherent in general-sum games and to provide mathematical tractability for studying learning at the intersection of games and matching markets. More specifically, our research is closely related to the work of \cite{o2021matrix}, who study learning in ZSGs under bandit feedback where agents additionally observe their opponents' actions. This information allows agents to form estimates of the underlying game matrix $A$ and apply optimism-based exploration strategies, leading to UCB-style algorithms. In self-play scenarios, their UCB approach empirically outperforms no-regret methods such as Exp3 \cite{Cesa-Bianchi_Lugosi_2006}, achieving faster convergence to the equilibrium. Crucially, their work evaluates performance using Nash regret, which compares the obtained rewards to the value of the game, rather than standard external regret. We adopt a similar information structure—observing the opponents' actions alongside bandit feedback—and propose a regret measure that reduces to Nash regret in the special case of ZSGs.
\section{Motivation and Real-World Applications}
\label{app:mot}

The motivation for modeling matching with strategic agents stems from the observation that, in many social and economic environments, agents receive payoffs only after interacting through their actions. In such scenarios, an agent’s preference depends not only on whom they are matched with but also on the subsequent actions of both parties. Consequently, agents must decide both whom to match with and what policy to follow once matched.

Our setting lies at the intersection of two-sided matching and algorithmic game theory, where the equilibrium of the market is generally considered the desired outcome, as it aligns the incentives of the agents. In our framework, an equilibrium represents a scenario in which no agent can strategically deviate from their chosen actions or match to increase their utility. A natural class of such scenarios involves matched agents engaging in tasks or games that require coordination or strategic decision-making.

For example, consider the job market, where employees are matched with employers. The agents in this setting interact in ways that directly affect their utilities. An employee may adopt a particular behavioral policy, such as the level of effort or time investment, while employers make decisions that shape the working environment, influencing both their own utility and that of the employee.

A more realistic scenario for our learning setting arises in digital platforms, where agents engage in repeated interactions. One example is ride-hailing services such as Uber, where a central platform matches drivers with users. In this context, both parties take actions that influence the outcome: users may decide how much to tip, while drivers choose how quickly to drive or how they engage with passengers during the ride. These decisions affect the utilities and experiences of both parties, making the setting inherently strategic with respect to both the matching and the subsequent actions.

In our paper, we take a step toward studying the complex learning dynamics of strategic agents in matching markets by focusing on a specific instantiation of a centralized bipartite matching model, where agents engage in zero-sum games. This setting enables us to derive initial results for this novel learning problem.
\section{Additional result on ZSG}
\label{app:ZSG}
We now derive a corollary for the ZSG setting with bandit feedback based on the results of \cite{o2021matrix}. In particular, the authors prove in Theorem 1 that if an agent follows a UCB strategy against any adversary, the worst-case Nash regret is bounded by:
\begin{equation*}
     R(T) = \mathbb{E}\left[ \sum_{t=1}^T V^{\star} - y_t^\top A x_t \right] \leq \widetilde{O}(\sqrt{mkT}).  
\end{equation*}
\begin{corollary}
In the zero-sum game (ZSG) setting with bandit feedback, if both agents follow the UCB algorithm of \cite{o2021matrix} in a self-play scenario, then the absolute Nash regret is bounded as follows:
\[
R_{abs}(T) = \mathbb{E}\left[ \sum_{t=1}^T \left| V^{\star} - y_t^\top A x_t \right| \right] \leq \tilde{O}(\sqrt{T \carI \carJ}).
\]
\end{corollary}
\begin{myproof}
Consider a column and a row player whose game matrices satisfy $A_x = -A_y^{\top} = A$. We denote the strategy of the column player by $x$, the strategy of the row player by $y$, and the corresponding equilibrium values by $V_x^\star$ and $V_y^\star$, where $V_x^\star = -V_y^\star = V^\star$.

The analysis of \cite{o2021matrix} suggests that, under the good event $\neg \mathcal{E}$, at each time step $t$, for a column player following the UCB method against an arbitrary opponent, we have that: $$ V_x^{\star} - \mathrm{E}^t[y_t^T A_x x_t] \leq \mathrm{E}^t[\sqrt{\frac{2}{1 \vee n^t_{i_t,j_t}}\log{\frac{1}{\delta}}}] $$

By symmetry, the same argument holds for the row player following UCB against any arbitrary opponent, so that: $$ V_y^{\star} - \mathrm{E}^t[x_t^T A_y y_t] \leq \mathrm{E}^t[\sqrt{\frac{2}{1 \vee n^t_{i_t,j_t}}\log{\frac{1}{\delta}}}] $$

Now, replacing $-V_y^\star = V_x^\star$ and $ A_y = - A_x^T$, we have that: $$ -\ \mathrm{E}^t[\sqrt{\frac{2}{1 \vee n^t_{i_t,j_t}}\log{\frac{1}{\delta}}}] \leq V_x^{\star} - \mathrm{E}^t[y_t^T A_x x_t] $$

Combining this with the inequality for the other agent implies that the per-step absolute difference is: $$ \mid V^{\star} - \mathrm{E}^t[y_t^T A_x x_t] \mid \leq \mathrm{E}^t[\sqrt{\frac{2}{1 \vee n^t_{i_t,j_t}}\log{\frac{1}{\delta}}}]. $$

Now, following similar arguments to \cite{o2021matrix}: $$ R_{abs}(T) = \mathrm{E}[ \sum_{t=1}^T \mid V_x^{\star} - y_t^T A_x x_t\mid] \leq \widetilde{O}(\sqrt{mkT}) $$
    
\end{myproof}

\section{Proof of Proposition~\ref{prop:1}}
\label{sec:app_prop}
For completeness and ease of reading, we restate Proposition~\ref{prop:1} together with the definition of matching instability of Section 3.

\setcounter{definition}{2}
\begin{definition}[Matching Instability]
The matching instability $\mathbf{MI}(m, X)$ of a matching \( \matching \) under mixed strategies $X = \{\strategy_{\agent}\}_{\agent \in \setAgents}$ is defined as:
\begin{equation*}
    \min_{s \in \mathbb{R}} \sum_{a \in \setAgents} s_a
\end{equation*}
subject to:
\noindent
\begin{align}
&\min(V^{\star}_{\player,\arm} - \util_{\player,\matching(\player)}(\strategy_{\player},\strategy_{\matching(\player)}) - s_{\player}, 
V^{\star}_{\arm,\player} - \util_{\arm,\matching(\arm)}(\strategy_{\arm},\strategy_{\matching(\arm)}) - s_{\arm}) \leq 0, \forall_{\player,\arm \in\setPlayer\times\setArms} \label{eq:C1}\tag{C1}\\
&\util_{\agent,\matching(\agent)}(\strategy_{\agent},\strategy_{\matching(\agent)}) - \util_{\agent,\bot} + s_{\agent} \geq 0, \; \forall \agent \in \setAgents \label{eq:C2}\tag{C2}\\
&V^{\star}_{\agent,\matching(\agent)} - \util_{\agent,\matching(\agent)}(\strategy_{\agent},\strategy_{m(\agent)}) - s_{\agent} \leq 0, \; \forall \agent \in \setAgents \label{eq:C3}\tag{C3}\\
&s_{\agent} \geq 0, \; \forall \agent \in \setAgents \label{eq:C4}\tag{C4}
\end{align}
\end{definition}

\setcounter{proposition}{0}
\begin{proposition}
The Matching Instability \( \mathbf{MI}(\matching, X) \) of a solution \( (\matching, X) \) satisfies the following properties:
\begin{enumerate}
    \item \( \mathbf{MI}(\matching, X) \) is always non-negative.
   \item \( \mathbf{MI}(\matching, X) \) equal to zero if and only if:
    \begin{itemize}
    \item[A)] Each agent receives the same valuation as in the NE of their respective game, i.e., 
    \( \util_{\agent,\matching(\agent)}(x_{\agent},x_{\matching(\agent)}) = V^{\star}_{\agent,\matching(\agent)} \forall \agent \in \setAgents \).
    \item[B)] Matching \( \matching \) is stable under the strategies $X$.
    \end{itemize}
    \item In the case where each agent has a single action, i.e., \( \carI = \carJ = 1 \), the model reduces to the two-sided matching problem of Section 2, and the matching instability coincides with the subset instability of Definition 1.
    \item In the case of a single agent on each side, i.e., \( \setPlayer = \{ \player \} \) and \( \setArms = \{ \arm \} \), where they are always matched, i.e., \( m(\player) = \arm \), the model reduces to the ZSG setting. In this case, the matching instability \( \mathbf{MI}(\matching, X) \) for strategies \( X = \{ \strategy_\player, \strategy_\arm \} \) becomes:
    \begin{equation*}
        \mathbf{MI}(\matching, X) = \left| V^{\star}_{\agent,\player} - U_{\agent, \player}(\strategy_\agent, \strategy_{\player}) \right|,
    \end{equation*}
    which measures the deviation of $(\strategy_\player,\strategy_\arm)$ from NE.
\end{enumerate}
\end{proposition}

\setcounter{proposition}{1}

\begin{myproof} \textit{ } \\
\proofstep{Case 1: Non-negativity} \\
By definition, the subset matching instability is always non-negative due to constraint \ref{eq:C4}.

\proofstep{Case 2:}  \\
We begin by showing that matching instability $\mathbf{MI}$ equals zero i.e., \( s_\agent = 0 \) for all \( \agent \in \setAgents \), then the matching $\matching$ is stable under the stategies $X$, and the utlities of the agents are the same with NE.

We start by showing that the satisfaction of the constraint \ref{eq:C3} ensures $ (\strategy_\agent, \strategy_{\matching(\agent)})$ has the same valuation with a Nash equilibrium for an arbitrary pair of matched agents $\agent$ and $\matching(\agent)$. To see this consider the constrain \ref{eq:C3} for the agents $\agent$ where we obtain that:
\begin{equation}
        V^{\star}_{\agent,\matching(\agent)} \leq \util_{\agent,\matching(\agent)}(\strategy_\agent,\strategy_{m(\agent)}) \label{c31}\\
\end{equation}
while the constrain for the agent $\matching(\agent)$ is:
\begin{equation}
    V^{\star}_{\matching(\agent),\agent} \leq \util_{\matching(\agent),\agent}(\strategy_{\matching(\agent)},\strategy_{\agent}) \label{c32}
\end{equation}
Expanding the inequality from \ref{c31} using the bilinear form we have that:
\begin{align*}
\strategy^{\star T}_{\agent} A_{\agent,\matching(\agent)} \strategy^{\star}_{\matching(\agent)} \leq \strategy_{\agent}^T A_{\agent,\matching(\agent)} \strategy_{\matching(\agent)} \tag{E1} \label{E1}
\end{align*}
Similarly, for the second inequality \ref{c32} :
\begin{align*}
    \strategy_{\matching(\agent)}^{\star T} A_{\matching(\agent),\agent} \strategy^{\star}_{\agent} &\leq \strategy_{\matching(\agent)}^T A_{\matching(\agent),\agent} \strategy_{\agent} \\
    \strategy_{\matching(\agent)}^{\star T} (-A_{\agent, m(\agent)})^T  \strategy^{\star}_{\agent}  &\leq \strategy_{\matching(\agent)}^T (-A_{\agent, \matching(\agent) })^T \strategy_{\agent} \\
    \strategy_{\matching(\agent)}^{\star T} A_{\agent, m(\agent)}^T  \strategy^{\star}_{\agent} &\geq  \strategy_{\matching(\agent)}^T A_{\agent, \matching(\agent) }^T \strategy_{\agent} \\
    \strategy^{\star T}_{\agent} A_{\agent,m(\agent)} \strategy_{\matching(\agent)}^{\star} &\geq \strategy_{\agent}^T A_{\agent,m(\agent)} \strategy_{\matching(\agent)} \tag{E2} \label{E2}
\end{align*}
Thus, from \ref{E1} and \ref{E2} we obtain:
\begin{align*}
    \strategy^{\star T}_{\agent} A_{\agent,m(\agent)} \strategy_{\matching(\agent)}^{\star} \leq &\strategy_{\agent}^T A_{\agent,\matching(\agent)} \strategy_{\matching(\agent)} \leq \strategy^{\star T}_{\agent} A_{\agent,m(\agent)} \strategy_{\matching(\agent)}^{\star}c \Rightarrow \\
    &\util_{\agent,m(\agent)}(\strategy_\agent, \strategy_{\matching(\agent)}) = V^\star_{\agent,m(\agent)}
\end{align*}
which implies that \( (\strategy_\agent, \strategy_{\matching(\agent)}) \) has the same valuation with the Nash equilibrium for the pair \( (\agent, \matching(\agent)) \).

Now, the constraints
\ref{eq:C1} and
\ref{eq:C2} imply that the matching
$\matching$ is stable under the strategies
$X$, as it has no blocking pair and is individually rational.

Conversely, if we have a matching equilibrium $(\matching, X)$, then \( s_\agent = 0 \; \forall \agent \in \setAgents\) is a feasible solution to the optimization constraints, and we have \( \mathbf{MI}(\matching, X) = 0 \), since the metric is always non-negative, that completing the proof.
\\

\proofstep{Case 3: Single Action for Each Agent}  \\
In the case where each agent has a single action, the constraint \ref{eq:C3} of Definition \ref{def:m_instability} is trivially satisfied, while the constraints \ref{eq:C1} and \ref{eq:C2} reduce to those in Definition~1 by omitting the strategy indices.

\proofstep{Case 4: Single pair of agents} \\
We now consider the case of a single pair of agents \( \player \) and \( \arm \) who always match, i.e., \( \matching(\player) = \arm \). This situation can occur in our matching model when agents strictly prefer matching over staying single, that is, when \( \util_{\player,\bot} = \util_{\arm,\bot} = -1 \).

In this setting, we obtain a zero-sum game where we denote the strategies $X = \{\strategy_\player, \strategy_\arm \}$ for the agents, where the matching instability becomes:
\begin{equation*}
    \min_{s \in \mathbb{R}} \sum_{a \in \setAgents} s_a
\end{equation*}
subject to:
\noindent
\begin{align}
&V^{\star}_{\player,\arm} - \util_{\player,\arm}(\strategy_{\player},\strategy_{\arm}) - s_{\player} \leq 0 \label{prop:c1}\\
&V^{\star}_{\arm,\player} - \util_{\arm,\player}(\strategy_{\arm},\strategy_{\player}) - s_{\arm} \leq 0 \label{prop:c2}\\
&s_\agent \geq 0 \quad \forall \agent \in \setAgents.
\end{align}
as constraints \ref{eq:C1} and \ref{eq:C2} of Definition~\ref{def:m_instability} are trivially satisfied in this case. More specifically, the blocking pairs constraint \ref{eq:C1} holds, as there is only a single pair \( (\player, \arm) \) of agents. Constraint \ref{eq:C2}, which enforces individual rationality, is also satisfied because both agents always prefer to be matched rather than remain single: \( \util_{\player,\bot} = \util_{\arm,\bot} = -1 \leq A \).

Using the symmetric properties of the ZSG, \( V^{\star}_{\player,\arm} = - V^{\star}_{\arm,\player} \) and \( U_{\player,\arm}(\strategy_{\player}, \strategy_{\arm}) = - U_{\arm,\player}(\strategy_{\arm}, \strategy_{\player}) \), the constraints in (\ref{prop:c1}) and (\ref{prop:c2}) can be written as:
\begin{equation*}
    -s_\arm \leq V^{\star}_{\player,\arm} - \util_{\player, \arm}(\strategy_{\player},\strategy_{\arm}) \leq s_\player \quad \textit{and} \quad  - s_\player \leq V^{\star}_{\arm, \player} - \util_{\arm, \player}(\strategy_{\arm},\strategy_{\player})\leq s_\arm
\end{equation*}
or equivalently:
\begin{equation*}
\mathbf{MI}(\matching, X) = \min_{s \in \mathbb{R}_{\geq 0}^{\mathcal{A}}} \left\{
\sum_{\agent \in \setAgents} s_{\agent} \;\middle|\;
\left| V^{\star}_{\agent, \matching(\agent)} - U_{\agent, \matching(\agent)}(\strategy_{\agent}, \strategy_{\matching(\agent)}) \right| \leq \max_{\agent \in \setAgents}(s_\agent) \quad \forall \agent \in \setAgents
\right\}
\end{equation*}

In addition, note that only one of the values $s_\player$ and $s_\arm$ can be positive, because of properties of the zero-sum game. So we let, $\mathfrak{s \in \mathbb{R}_{\geq 0}}$ be the positive subsidy we have that:
\begin{equation*}
\mathbf{MI}(\matching, X) = \min_{ \mathfrak{s} \in \mathbb{R}_{\geq 0}} \left\{
\mathfrak{s} \;\middle|\;
\left| V^{\star}_{\agent, \matching(\agent)} - U_{\agent, \matching(\agent)}(\strategy_{\agent}, \strategy_{\matching(\agent)}) \right| \leq \mathfrak{s} \quad \forall \agent \in \setAgents
\right\}
\end{equation*}
and because the values $\left| V^{\star}_{\agent, \matching(\agent)} - U_{\agent, \matching(\agent)}(\strategy_{\agent}, \strategy_{\matching(\agent)}) \right|$ are equal for every agent $\agent \in \setAgents$ we have:
\begin{align}
\mathbf{MI}(\matching, X) =
\left| V^{\star}_{\agent, \matching(\agent)} - U_{\agent, \matching(\agent)}(\strategy_{\agent}, \strategy_{\matching(\agent)}) \right| \quad \forall \agent \in \setAgents
\end{align}
\end{myproof}

\section{Proof of Theorem~\ref{theo:ucb}}
\label{sec:app_proof}

In this section, we provide an upper bound for Algorithm 1 in the self-play scenario, where all agents follow the principle of optimism in the face of uncertainty, as described in Section 4.

More specifically, each agent \( \agent \in \setAgents \) maintains an upper confidence estimate \( \bar{A}_{\agent, \agent'} \) for matching with an agent \(\agent' \in OS(\agent)\) of the other side of the market. At each step \( t \), agents estimate their preferences \( \hat{\pref}_{\agent} \) by sorting the estimated values of the different games, i.e., \( \hat{\pref}_\agent = \argsort_{\agent'\in OS(\agent)} \bar{V}^{\star}_{\agent, \agent'} \), and report them to the platform. The platform selects matchings \( \matching_t \) using the GS algorithm based on the agents' estimated preferences. Then, the matched agents select the min-max strategies with respect to the upper confidence payoff matrix of the game $\bar{A}_{\agent, \matching(\agent)}$.

Next, we present Proposition~\ref{prop:2}, which describes key properties of our Algorithm 1 that we will use in the proof of our main theorem.
\setcounter{proposition}{1}
\begin{proposition}
    \label{prop:2}
    Let $\matching_t$ be the matching and $X = \{ \strategy_\agent \}_{\agent \in \setAgents}$  the strategies of the agents under Algorithm 1 at time step $t$. Then the following statements hold:
    \begin{enumerate}[label=(S\arabic*), ref=S\arabic*]
        \item \label{eq:S1} $\bar{V}^{\star}_{\agent,\matching(\agent)} - \bar{\util}_{\agent,\matching(\agent)}(\strategy_{\agent}, \strategy_{\matching(\agent)}) \leq 0 \; \forall \agent \in \setAgents $
        \item \label{eq:S2}
        $\min( \bar{V}_{\player,\arm}^\star - \bar{\util}_{\player,\matching(\player)}(\strategy_\player, \strategy_{\matching(p)}),
              \bar{V}_{\arm, \player}^\star - \bar{\util}_{\arm,\matching(\arm)}(\strategy_\arm, \strategy_{\matching(\arm)})) \leq  0   \quad \forall (\player, \arm) \in \setPlayer \times \setArms$
    \end{enumerate}
where \( \bar{V}_{\agent, \agent'}^\star = \bar{\util}_{\agent,\agent'}(\strategy_\agent^\star, \strategy_{\agent'}^\star) = \strategy_\agent^{\star T} \bar{A}_{\agent, \agent'} \strategy_{\agent'}^\star \), is the value of the game for the pair of agents \((\agent, \agent')\), with respect to the upper confidence payoff matrix \( \bar{A}_{\agent, \agent'} \).
\end{proposition}
\setcounter{proposition}{2}

\begin{myproof}
At each step  $t$, the platform selects a matching $\matching_t$ and the matching pairs $(\player, \arm) \in \matching_t$ select strategies $\strategy_{\player}^t, \strategy_{\arm}^t$. In the following, we omit the indexing for the time step $t$ for the convenience of notation.

\proofstep{myproof of statement \ref{eq:S1}:} \\
More specifically, each agent selects their min-max strategy with respect to their upper confidence payoff matrix. For a pair \( (\player, \arm) \), this results in:
\begin{enumerate}
    \item Agent \( \player \) selects \( \strategy_\player = x^{\star}_\player = \arg\max_{\strategy} \min_{y} \strategy^T \bar{A}_{\player,\arm} y \). We denote by \( (x^{\star}_\player, y^{\star}_\arm) \) the Nash equilibrium of the game with payoff matrix \( \bar{A}_{\player,\arm} \), and by \( \bar{V}^{\star}_{\player,\arm} = x^{\star T}_\player \bar{A}_{\player,\arm} y^{\star}_\arm \) the value of the game.

    \item Accordingly, agent \( \arm \) selects \( \strategy_\arm = x^{\star}_\arm = \arg\max_{\strategy} \min_{y} \strategy^T \bar{A}_{\arm,\player} y \), with \( (x^{\star}_\arm, y^{\star}_\player) \) being the Nash equilibrium of the game with payoff matrix \( \bar{A}_{\arm,\player} \). The value of the game for agent $\arm$ is $\bar{V}^{\star}_{\arm,\player} = x^{\star T}_\arm \bar{A}_{\arm,\player} y^{\star}_\player $
\end{enumerate}

Note that the pair of strategies \( (x_\player = x^{\star}_\player, x_\arm = x^{\star}_\arm) \) selected by the agents is not necessarily equal to the Nash equilibrium strategy profiles \( (x^{\star}_\player, y^{\star}_\arm) \) and \( (x^{\star}_\arm, y^{\star}_\player) \) of their respective games.

Instead, using the property of the ZSG, where $\util(\strategy^\star, y) \leq \util(\strategy^\star, y^\star) \forall y \in \Delta_{y}$, we have that:
\begin{equation*}
    \bar{V}^{\star}_{\player,\arm} - \bar{\util}_{\player,\arm}(\strategy_{\player}, \strategy_{\arm}) =
    \bar{\util}_{\player,\arm}(\strategy_{\player}^{\star}, y_{\arm}^{\star}) - \bar{\util}_{\player,\arm}(\strategy_{\player}, \strategy_{\arm}) =
    \bar{\util}_{\player,\arm}(\strategy_{\player}^{\star}, y_{\arm}^{\star}) - \bar{\util}_{\player,\arm}(\strategy_{\player}^{\star}, \strategy_{\arm}) \leq 0
\end{equation*}
as \( \strategy_\player = \strategy^\star_\player \), and similarly for agent \( \arm \), we have:
\begin{equation*}
     \bar{V}^{\star}_{\arm,\player} - \bar{\util}_{\arm,\player}(\strategy_{\arm}, \strategy_{\player}) =
     \bar{\util}_{\arm,\player}(\strategy_{\arm}^{\star}, y_{\player}^{\star})  - \bar{\util}_{\arm,\player}(\strategy_{\arm}^{\star}, \strategy_{\player}) \leq 0
\end{equation*}
So, for every agent $\agent \in \setAgents$, using the notation of the matching $\matching$, we have:
\begin{equation*}
    \bar{V}^{\star}_{\agent,\matching(\agent)} - \bar{\util}_{\agent,\matching(\agent)}(\strategy_{\agent}, \strategy_{\matching(\agent)}) \leq 0 \quad \forall \agent \in \setAgents.
\end{equation*}

\proofstep{myproof of statement \ref{eq:S2}} \\
Now we prove that the solution $(\matching, X)$ has no blocking pair, with respect to the upper confidence estimate of the payoff matrices $\bar{A}_{\agent,\matching(\agent)}$, i.e.:
\begin{align*}
\min( \bar{V}_{\player,\arm}^\star - \bar{\util}_{\player,\matching(\player)}(\strategy_\player, \strategy_{\matching(\player)}),
      \bar{V}_{\arm,\player}^\star - \bar{\util}_{\arm,\matching(\agent)}(\strategy_\arm, \strategy_{\matching(\agent)})) \leq  0  \quad \forall (p, a) \in \setPlayer \times \setArms
\end{align*}

More specifically, we have that:
\begin{align*}
&\min( \bar{V}_{\player,\arm}^\star - \bar{\util}_{\player,\matching(\player)}(\strategy_\player, \strategy_{\matching(\player)}),
      \bar{V}_{\arm,\player}^\star - \bar{\util}_{\arm,\matching(\agent)}(\strategy_\arm, \strategy_{\matching(\agent)})) \leq  \\
&\min(\bar{V}_{\player,\arm}^\star - \bar{V}_{\player,\matching(\player)}^\star,
     \bar{V}_{\arm,\player}^\star - \bar{V}_{\arm,\matching(\agent)}^\star) \leq 0 \quad \forall (p, a) \in \setPlayer \times \setArms
\end{align*}
The first inequality follows from the statement \ref{eq:S1}, where $ - \bar{\util}_{\agent,\matching(\agent)}(\strategy_{\agent}, \strategy_{\matching(\agent)}) \leq -\bar{V}^{\star}_{\agent,\matching(\agent)} \; \forall \agent \in \setAgents$.
The second inequality holds, from the fact that in each step, the central platform selects a stable matching $\matching$, by running the GS algorithm using preferences that are formed according to the upper confidence values of the games $\bar{V}^{\star}$, so there are no blocking pairs.
\end{myproof}

\subsection{Per step regret}
We now proceed with the regret analysis of the algorithm. We begin by examining the regret at a single time step \( t \) under the event that true utilities lies between the confidence intervals of the estimated values, as formalized in the following lemma.

First, let \( \bar{A}^t_{\agent, \agent'}(i,j) \) and \( \underbar{A}^t_{\agent, \agent'}(i,j) \) denote the upper and lower confidence bounds, respectively, for the entries of the true payoff matrix \( A_{\agent, \agent'} \), as defined below:
\begin{equation*}
    \bar{A}^t_{\agent, \agent'}(i,j) = \hat{A}^t_{\agent, \agent'}(i,j) + \sqrt{\frac{2 \log(1/\delta)}{1 \vee n_{\agent, \agent'}^t(i,j)}} \quad \forall \agent \in \setAgents, \; \agent' \in OS(\agent),
\end{equation*}
\begin{equation*}
    \underbar{A}^t_{\agent, \agent'}(i,j) = \hat{A}^t_{\agent, \agent'}(i,j) - \sqrt{\frac{2 \log(1/\delta)}{1 \vee n_{\agent, \agent'}^t(i,j)}} \quad \forall \agent \in \setAgents, \; \agent' \in OS(\agent).
\end{equation*}

\begin{Lemma}[Per-step Regret]
Consider the event where the entries of the true payoff matrices lie within the confidence intervals \( C_{\agent, \agent'}(i,j) = [\underbar{A}_{\agent,\agent'}(i,j), \bar{A}_{\agent,\agent'}(i,j)] \), i.e.:
\begin{equation}
    \mathcal{E}_t = \left\{ A_{\agent, \agent'}(i,j) \in C_{\agent, \agent'}(i,j) \; \forall \agent, \agent' \in \setAgents \times \setAgents,\; \forall i, j \in \rowSetAction \times\columnSetAction \right\}.
\end{equation}
Assuming that the event  $\mathcal{E}_t$ holds the matching instability of the solution \( (\matching_t, X_t) \) of Algorithm 1 at time \( t \) is bounded by

\[
\mathbf{MI}(\matching_t, X_t) \leq \sum_{\agent \in \setAgents} \mathbb{E}\left[\sqrt{\frac{2}{1 \vee n^t_{\agent,\matching_t(\agent)}(i^t_{\agent}, i^t_{\matching_t(\agent)})} \log \left(\frac{2}{\delta}\right)}\mid i^t_{\agent} \sim \strategy^t_{\agent}, i^t_{\matching_t(\agent)} \sim \strategy^t_{\matching_t(\agent)} \right].
\]
\end{Lemma}
\begin{myproof}
In the following, we omit the time index \( t \) from the strategies \( \strategy \) and the matching \( \matching \) for convenience.

\proofstep{Step 1. Per-step Regret} \\
We bound the per-step regret of the algorithm under the $\mathcal{E}_t$, by constructing subsidies $s_\agent$, which we prove to be feasible.

More specifically we define $s_\agent$ as follows :
\begin{equation}
\label{eq:sa}
    s_\agent = \bar{\util}_{\agent,\matching(\agent)}(\strategy_\agent, \strategy_{\matching(\agent)}) - \util_{\agent,\matching(\agent)}(\strategy_\agent, \strategy_{\matching(\agent)}) \; \forall \agent \in \setAgents
\end{equation}
where \( \strategy_\agent \) are the strategies chosen by the agents, and \( \bar{\util} \) denotes the expected utility of the selected strategies with respect to the upper confidence matrix \( \bar{A} \).

To prove that $s = \{s_\agent\}_{\agent \in \setAgents}$ is a feasible solution to the linear program we show that the constrains \ref{eq:C1}, \ref{eq:C2}, \ref{eq:C3} and \ref{eq:C4} are satisfied.

\proofstep{Satisfaction of constrain \ref{eq:C4}} \\
Due to the optimism we have that $\bar{\util}_{\agent,\matching(\agent)} \geq \util_{\agent,\matching(\agent)}$, and thus
\begin{equation}
    s_{\agent} =  \bar{\util}_{\agent,\matching(\agent)}(\strategy_\agent, \strategy_{\matching(\agent)}) - \util_{\agent,\matching(\agent)}(\strategy_\agent, \strategy_{\matching(\agent)}) \geq 0 \; \forall \agent \in \setAgents
\end{equation}
satisfying constrain \ref{eq:C4}.

\proofstep{Satisfaction of constrain \ref{eq:C3}} \\
We continue with the Nash rationality constrains \ref{eq:C3} for the agents $\agent \in \setAgents$ were we have that:
\begin{align}
  &V^{\star}_{\agent,\matching(\agent)} - \util_{\agent,\matching(\agent)}(\strategy_\agent,\strategy_{\matching(\agent)}) - s_\agent = \\
  &V^{\star}_{\agent,\matching(\agent)} - \util_{\agent,\matching(\agent)}(\strategy_\agent,\strategy_{\matching(\agent)}) - \bar{\util}_{\agent,\matching(\agent)}(\strategy_\agent,\strategy_{\matching(\agent)}) + \util_{\agent,\matching(\agent)}(\strategy_\agent,\strategy_{\matching(\agent)}) = \\
  &V^{\star}_{\agent,\matching(\agent)} - \bar{\util}_{\agent,\matching(\agent)}(\strategy_\agent,\strategy_{\matching(\agent)}) \leq \\
  &\bar{V}^{\star}_{\agent,\matching(\agent)} - \bar{\util}_{\agent,\matching(\agent)}(\strategy_\agent,\strategy_{\matching(\agent)}) \leq 0
\end{align}
In the first equality, we replace \( s_\agent \) according to our definition in equation~(\ref{eq:sa}). The first inequality follows from the fact the $V^{\star}_{\agent,\matching(\agent)} \leq \bar{V}^{\star}_{\agent,\matching(\agent)}$ due to the optimism, and the second from the previews argument \ref{eq:S1} of Proposition \ref{prop:2}. So our subsidies \( s_\agent \) satisfy the Nash rationality constraint~\eqref{eq:C3}.

\proofstep{Satisfaction of constrain~\eqref{eq:C2}}\\
For the individual rationality of constrains~\eqref{eq:C2} we have that:
\begin{align*}
    &\util_{\agent,\matching(\agent)}(\strategy_\agent,\strategy_{\matching(\agent)}) - \util_{\agent,\bot} + s_\agent  = \\
    &\util_{\agent,\matching(\agent)}(\strategy_\agent,\strategy_{\matching(\agent)}) - \util_{\agent,\bot} + \bar{\util}_{\agent,\matching(\agent)}(\strategy_\agent, \strategy_{\matching(\agent)}) - \util_{\agent,\matching(\agent)}(\strategy_\agent, \strategy_{\matching(\agent)}) = \\
    &\bar{\util}_{\agent,\matching(\agent)}(\strategy_\agent, \strategy_{\matching(\agent)}) - \util_{\agent,\bot} \geq \\
    &\bar{V}^{\star}_{\agent,\matching(\agent)} - \util_{\agent,\bot}\geq 0
\end{align*}
The first inequality follows from the argument \ref{eq:S1} of Proposition \ref{prop:2}.
The last inequality follows from the fact that the matching \( \matching \) is stable under \( \strategy_{\agent} \) according to the upper confidence bounds $\bar{V}^{\star}$, and therefore satisfies individual rationality.

\proofstep{Satisfaction of constrain \ref{eq:C1}}\\
We now continue with the blocking pair constraint, so for~\eqref{eq:C1}, we have:
\begin{align}
    & \min\left( V^{\star}_{\player,\arm} - \util_{\player,\matching(\player)}(\strategy_\player,\strategy_{\matching(\player)})  - s_\player, \; V^{\star}_{\arm,\player} - \util_{\arm,\matching(\arm)}(\strategy_\arm,\strategy_{\matching(\arm)})  - s_\arm \right) =  \\
    & \min\left( V^{\star}_{\player,\arm} - \bar{\util}_{\player,\matching(\player)}(\strategy_\player,\strategy_{\matching(\player)}) , \; V^{\star}_{\arm,\player} - \bar{\util}_{\arm,\matching(\arm)}(\strategy_\arm,\strategy_{\matching(\arm)})\right) \leq  \\
    & \min\left( \bar{V}^{\star}_{\player,\arm} - \bar{\util}_{\player,\matching(\player)}(\strategy_\player,\strategy_{\matching(\player)}) , \; V^{\star}_{\arm,\player} - \bar{\util}_{\arm,\matching(\arm)}(\strategy_\arm,\strategy_{\matching(\arm)})\right) \leq 0
\end{align}

Where the first inequality follows from the optimism $V^{\star}_{\agent,\matching(\agent)} \leq \bar{V}^{\star}_{\agent,\matching(\agent)}$, and the second from the previews argument \ref{eq:S2} of Proposition \ref{prop:2}.

\proofstep{Bounding per step regret}\\
Therefore, since \( s_\agent \) is a feasible solution to the problem, the instability of the market can be bounded as follows:
\begin{align}
    \mathbf{MI}(\matching_t, X_t) &\leq \sum_{\agent \in \setAgents} s_\agent  \\
    &=\sum_{\agent \in \setAgents} \bar{\util}_{\agent,\matching(\agent)}(\strategy_\agent, \strategy_{\matching(\agent)}) - \util_{\agent,\matching(\agent)}(\strategy_\agent, \strategy_{\matching(\agent)})  \\
    &=\sum_{\agent \in \setAgents} \strategy_\agent^T\bar{A}_{\agent,\matching(\agent)}\strategy_{\matching(\agent)} - \strategy_\agent^TA_{\agent,\matching(\agent)}\strategy_{\matching(\agent)}   \\
    &\leq \sum_{\agent \in \setAgents} \strategy_\agent^T\bar{A}_{\agent,\matching(\agent)}\strategy_{\matching(\agent)} - \strategy_\agent^T\underbar{A}_{\agent,\matching(\agent)}\strategy_{\matching(\agent)}   \\
    &=\sum_{\agent \in \setAgents} \mathbb{E}\left[2\sqrt{\frac{2}{1 \vee n_{\agent,\matching(\agent)}(i_{\agent}, i_{\matching(\agent)})} \log \left(\frac{1}{\delta}\right)}\mid i_{\agent} \sim \strategy_{\agent}, i_{\matching(\agent)} \sim \strategy_{\matching(\agent)} \right]
\end{align}
where for the inequality, we use the fact that under the event \( \mathcal{E}_t \), it holds that \( A_{\agent,\matching(\agent)} \geq \underline{A}_{\agent,\matching(\agent)} \).\\
\end{myproof}
\setcounter{theorem}{0}

\subsection{Main theorem}
Now we prove our main Theorem~\ref{theo:ucb}, that we restate bellow:
\setcounter{theorem}{0}
\begin{theorem}
 Let $\mid \setPlayer \mid = \carP$, $\mid \setArms \mid = \carA$, $\mid \rowSetAction \mid = \carI$ and $\mid \columnSetAction \mid = \carJ$. When Algorithm 1 is run with parameter $\delta= 1/(4  T^2  \carP^2\carA^2 \carI \carJ)$, the expected regret $R_T$ under horizon $T$ is bounded:
 \begin{equation*}
    \mathbb{E}[R_T] \leq O\left( 2 \sqrt{4 T \carI \carJ \carP \carA   \log\left(4 T^2 \carI \carJ \carP^2 \carA^2  \right)} \right)
\end{equation*}
\end{theorem}
\begin{myproof} To bound the expected regret, we consider the event where the entries of the true payoff matrices lie within the confidence intervals \( C^t_{\agent, \agent'}(i,j) = [\underbar{A}^t_{\agent,\agent'}(i,j), \bar{A}^t_{\agent,\agent'}(i,j)] \) for all time steps, i.e.:
\begin{equation}
\label{event_all}
\mathcal{E} =\left\{\mathcal{E}_t \; \forall \; t \in [T] \right\}= \left\{ A_{\agent, \agent'}(i,j) \in C^t_{\agent, \agent'}(i,j) \; \forall \agent, \agent' \in \setAgents \times \setAgents,\; \forall i, j \in \rowSetAction \times\columnSetAction, \forall t \in [T]\right\}.
\end{equation}
So for the expected regret, we have that:
\begin{align*}
    \mathbb{E}[R_T] &=\mathbb{E}[R_T\mid \mathcal{E} ] \Pr( \mathcal{E}) + \mathbb{E}[R_T\mid \neg \mathcal{E}] \Pr(\neg \mathcal{E}) \\
    &\leq \mathbb{E}[R_T\mid \mathcal{E} ] + \mathbb{E}[R_T\mid \neg \mathcal{E}] \Pr(\neg \mathcal{E}) \\
    &\leq  \sum_{t=1}^T\underbrace{\mathbf{MI}(m_t, X_t)}_{(A)}  + \underbrace{2  T  (\carA  + \carP)  \Pr(\neg \mathcal{E})}_{(B)}
\end{align*}
As the regret $\mathbb{E}[R_T\mid \neg \mathcal{E}]$ under $\neg \mathcal{E}$ can be bounded by, $2T(\carA  + \carP)$, in the worst-case.

The event \( \mathcal{E} \) involves \( O(T \times \carP \times \carA \times \carI \times \carJ) \) independent 1-sub-Gaussian random variables. Therefore, by applying the union bound and the concentration bounds for sub-Gaussian (see also Appendix \ref{app:sg_lemma}) random variables, we can bound \( \Pr(\neg \mathcal{E}) \) as follows:
\begin{equation*}
    \Pr(\neg \mathcal{E}) \leq 2 \delta T  \carP   \carA \carI  \carJ   =
    \frac{1}{2 T  \carP  \carA },
\end{equation*}
by setting $\delta = \frac{1}{4 T^2  \carP^2 \carA^2  \carI \carJ}$.

For the second term (B) we have that:
\begin{equation}
    (B) \leq 2  T  (\carP +  \carA)   \Pr(\neg \mathcal{E}) \leq 4  T  \carP  \carA  \Pr(\neg \mathcal{E}) \leq 2
\end{equation}
using that $(a + b) \leq 2ab$ for $a,b \geq 1$.

For the first term (A) we have that:
\begin{align}
    (A) &= \sum_{t=1}^T\mathbf{MI}(\matching_t, X_t) \leq \sum_{t=1}^T \sum_{\agent \in \setAgents} \mathbb{E}\left[\sqrt{\frac{2}{1 \vee n^t_{\agent,\matching_t(\agent)}(i^t_{\agent},i^t_{\matching_t(\agent)})}\log \left(\frac{1}{\delta}\right)}\right] \label{eqA:1}\\
    &\leq 2 \sum_{t=1}^T \sum_{(\player,\arm) \in \matching_t}  \mathbb{E}\left[\sqrt{\frac{2}{1 \vee n^t_{\player,\arm}(i^t_{\player},i^t_{\arm})}\log \left(\frac{1}{\delta}\right)}\right]\label{eqA:2}\\
    &= 2 \sum_{i, j} \sum_{\player,\arm}  \mathbb{E}\left[\sum_{\substack{t=1 :\\ i^t_{\player} = i, i^t_{\arm}=j \\ (\player,\arm) \in m_t}}^T
   \sqrt{\frac{2}{1 \vee n^t_{\player,\arm}(i,j)}\log \left(\frac{1}{\delta}\right)}\right]\label{eqA:3}\\
    &\leq 2 \sum_{i, j} \sum_{\player,\arm} \mathbb{E}\left[ \sqrt{4 n^T_{\player,\arm}(i,j)\log \left(\frac{1}{\delta}\right)} \right]\label{eqA:4}\\
    &\leq 2 \sqrt{4 T \carI\carJ\carP\carA\log \left(\frac{1}{\delta}\right)} \label{eqA:5}
\end{align}
In \eqref{eqA:2}, we rewrite the summation by explicitly considering the matched agents. In \eqref{eqA:3}, we restructure the sum to account for all occurrences where the matched agents ($\player$,$\arm$) select actions ($\rowAction$,$\columnAction$). The transition from \eqref{eqA:3} to \eqref{eqA:4} follows from an integral bound, which is a consequence of the Fundamental Theorem of Calculus. Finally, in \eqref{eqA:5}, we obtain an upper bound using the Cauchy Schwarz inequality.

Thus the overall regret by setting $\delta = 1/(4  T^2  \carP^2 \carA^2  \carI \carJ)$ is bounded by:
\begin{align}
    R_T &\leq 2 \sqrt{4 T \carI \carJ \carP \carA   \log\left(4 T^2 \carI \carJ \carP^2 \carA^2  \right)} + 2\\
    &\leq \tilde{O}\left(\sqrt{T \carI \carJ \carP \carA } \right)
\end{align}
\end{myproof}

\subsection{Concentration bound for sub-Gaussian random variables}
\label{app:sg_lemma}
For completeness, we present a lemma on the concentration bounds for sub-Gaussian random variables. More details can be found in book of \cite{noauthororeditor}.

\begin{Lemma}
\label{lemma:subgaussian}
Let \( X_1, \ldots, X_n \) be i.i.d.\ \(\sigma\)-sub-Gaussian random variables with \( \mathbb{E}[X_i] = \mu \). Then, for all \( \epsilon \geq 0 \), it holds that:
\begin{equation*}
    \Pr\left( \left| \frac{1}{n} \sum_{i=1}^n X_i - \mu \right| > \epsilon \right) \leq 2 \exp\left(-\frac{n \epsilon^2}{2 \sigma^2}\right).
\end{equation*}
\end{Lemma}

Applying Lemma~\ref{lemma:subgaussian} for \(\sigma = 1\) and setting \( \epsilon = \sqrt{\frac{2 \ln(1/\delta)}{n}} \), we obtain:
\begin{equation*}
    \Pr\left( \left| \frac{1}{n} \sum_{i=1}^n X_i - \mu \right| > \sqrt{\frac{2 \ln(1/\delta)}{n}} \right) \leq 2 \delta.
\end{equation*}

Finally, using the union bound, the probability of the event \( \neg \mathcal{E} \) in Equation~\ref{event_all} is bounded as follows:
\begin{equation*}
    \Pr(\neg \mathcal{E}) \leq 2 \delta T \carP \carA \carI \carJ.
\end{equation*}

\section{Simulation details}
\label{app:exp}
We now present a set of simulations that complement our theoretical results. For each setting, we generate 50 random instances by independently sampling the entities payoff matrices of the ZSG from a normal distribution, i.e., $A_{p,a}(\rowAction,\columnAction) \sim \mathcal{N}(0,1)$ for all $(\player, \arm, \rowAction, \columnAction) \in \setPlayer \times \setArms \times \rowSetAction \times \columnSetAction$. We compares the regret for the different settings for varying numbers of agents and actions, along with the theoretical bound of $O\left(\sqrt{T\carA\carP\carI\carJ}\right)$ (dashed blue), where we omit the logarithmic term for clarity.

We compare the \textbf{self-play} scenario outlined in Algorithm~1 with two additional cases in which agents on one side, $\setArms$, have access to additional information and can act differently, in order to study how the regret scales relative to these settings. First, we consider the \textbf{Nash-response} setting. The agents in $\setPlayer$ follow Algorithm~1, while the agents in $\setArms$ form their preferences and select strategies according to the NE of the true payoff matrices \(A\). In this case, we expect the regret to be lower, as the agents on one side know the preferences and play according to the NE. Next, we consider the other extreme, the \textbf{Best-response} setting, where agents \( \arm \in \setArms \) know the true payoff matrices and additionally observe the strategies selected by agents in \( \setPlayer \), i.e., \( x_{\player,\arm} \) for all \( \player \in \setPlayer \), before selecting their own strategies. They then form their preferences and choose strategies that exploit their match by playing best responses. Specifically, each agent computes their best-response strategy as \( \strategy_{\arm,\player} = \argmax_{\strategy} \util_{\arm,\player}(\strategy, \strategy_{\player,\arm}) \), and form their preferences according to the resulting payoffs: \( \pref_{\arm} = \argsort_{\player \in \setPlayer} \util_{\arm,\player}(\strategy_{\arm,\player}, \strategy_{\player,\arm}) \).

Figure~\ref{fig:regret_best_response_all} presents all our experiments for the different configuration of agents and actions. In all experimental settings, we observe sublinear regret growth, indicating that the system progressively learns a \emph{matching equilibrium}. The \textbf{Nash-response} (dotted orange) scenario yields the lowest regret, as one side of the market has full knowledge of the payoff matrices and acts according to NE. In addition, the \textbf{Best-response} (dash-dotted purple) scenario initially outperforms \textbf{Self-play} (solid green), but as the number of agents and actions increases, its performance deteriorates. This is expected as agents in \( \setArms \) exploit their matched partners in \( \setPlayer \), which introduces greater complexity into the learning process.

\begin{figure}[ht]
\centering
\includegraphics[width=1\textwidth]{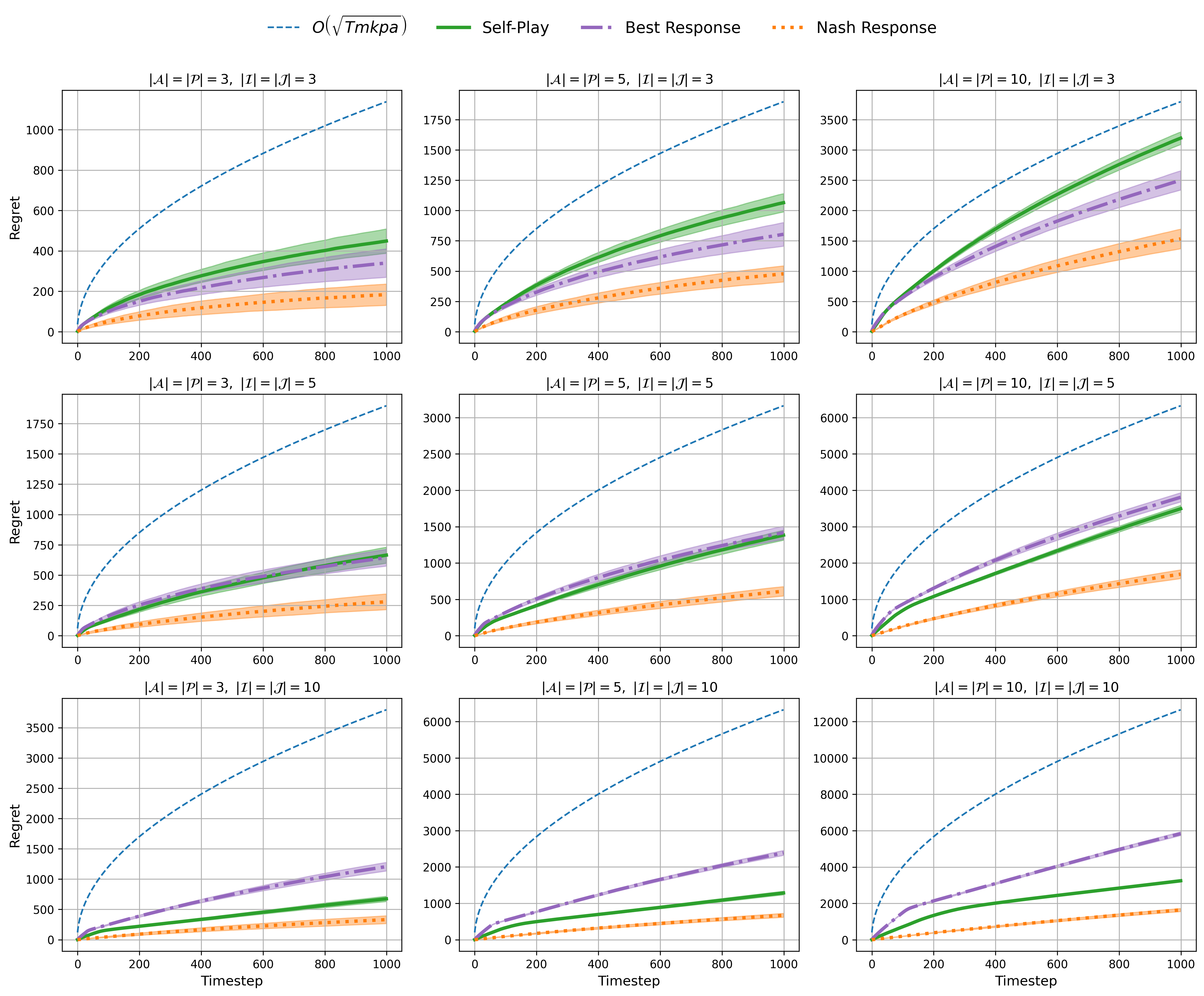}
\caption{ Regret of the different settings, including the theoretical bound for varying numbers of agents and actions. Lines represent the average performance of the algorithm over 50 runs and shaded areas indicate the standard deviation.}
\label{fig:regret_best_response_all}
\end{figure}

\end{document}